\documentclass[11pt]{article}
\usepackage{fullpage,times}
\usepackage{parskip}
\usepackage{url}            %
\usepackage{booktabs}       %
\usepackage{amsfonts}       %
\usepackage{nicefrac}       %
\usepackage{microtype}      %
\usepackage[font=small,labelfont=bf,tableposition=top]{caption}
\DeclareCaptionLabelFormat{andtable}{#1~#2  \&  \tablename~1}
\DeclareCaptionLabelFormat{empty}{}
\newcommand{\negaspace}{\vspace{-.5\baselineskip}}

\usepackage{amsmath, amsthm, amssymb, amsfonts, mathtools, graphicx, enumerate}
\usepackage{adjustbox}
\usepackage[utf8]{inputenc} 
\usepackage[T1]{fontenc}    
\usepackage{booktabs}
\usepackage{nicefrac}       
\usepackage{microtype}      
\usepackage{xspace}
\usepackage{algorithm,algorithmic}
\usepackage{color}
\usepackage{enumitem}
\usepackage{comment}
\usepackage{bm}
\usepackage{apptools}
\usepackage[page, header]{appendix}
\usepackage{titletoc}
\usepackage{array}

\usepackage[scaled=.9]{helvet}

\usepackage{url}
\usepackage{caption}
\usepackage[table,dvipsnames]{xcolor}
\usepackage[colorlinks=true, linkcolor=blue, citecolor=blue]{hyperref}

\usepackage{natbib}
\bibpunct{(}{)}{;}{a}{,}{,}

\usepackage{amsthm}   
\begingroup
    \makeatletter
    \@for\theoremstyle:=definition,remark,plain\do{%
        \expandafter\g@addto@macro\csname th@\theoremstyle\endcsname{%
            \addtolength\thm@preskip\parskip
            }%
        }
\endgroup

\newcommand{\yrm}[1]{}

\newcommand{\qiangremoved}[1]{}

 \def\bb#1\ee{\begin{align*}#1\end{align*}}

 \def\bba#1\eea{\begin{align}#1\end{align}}

\usepackage{lipsum}  
\makeatletter
\newcommand{\printfnsymbol}[1]{%
  \textsuperscript{\@fnsymbol{#1}}%
}
\makeatother

\title{\huge Let us Build Bridges: 
Understanding and Extending Diffusion Generative Models }
\author{%
	Xingchao Liu\\
    ~University of Texas at Austin\\
	\texttt{xcliu@cs.utexas.edu}
	\and
	Lemeng Wu\\
	~University of Texas at Austin\\
	\texttt{lmwu@cs.utexas.edu} \\
	\and 
	Mao Ye\\
	University of Texas at Austin\\
	\texttt{my21@cs.utexas.edu} \\
	\and
	Qiang Liu\\
	University of Texas at Austin\\
	\texttt{lqiang@cs.utexas.edu} \\
}

\date{}

\usepackage{qiangstyle} 
\begin{document}
\maketitle

\begin{abstract}
Diffusion-based generative models 
have achieved promising results recently, 
but raise an array of open questions in terms of 
conceptual understanding, 
theoretical analysis, algorithm improvement and extensions to discrete, structured, non-Euclidean domains.  %
This work tries
to re-exam the overall framework, in order to 
gain better theoretical understandings and 
develop algorithmic extensions for data from arbitrary domains.      
By viewing 
diffusion models 
as latent variable models with unobserved diffusion trajectories  
and applying maximum likelihood estimation (MLE) %
with latent trajectories imputed from an auxiliary distribution, 
we show that both the model construction and the imputation of latent trajectories amount to constructing diffusion bridge processes that achieve deterministic values and constraints at end point, %
for which we provide a systematic study and a suit of tools.  
Leveraging our framework, we present 1) a first  theoretical error analysis for learning diffusion generation models, 
and 2) a simple and unified approach to learning on data from different discrete and constrained domains. 
Experiments show that our methods perform superbly 
on generating images, semantic segments and 3D point clouds. 
\end{abstract}

\section{Introduction}
\negaspace
 Diffusion-based deep generative models,  
notably score matching with Langevin dynamics (SMLD)  \citep{song2019generative, song2020improved}, 
denoising diffusion probabilistic models (DDPM)  \citep{ho2020denoising}, 
and their variants  \citep[e.g.,][]{song2020score,  song2020denoising, kong2021fast, song2021maximum, nichol2021improved}, have shown to %
achieve new state of the art results for image synthesis \citep{dhariwal2021diffusion, ramesh2022hierarchical, ho2022cascaded, liu2021sampling}, audio synthesis~\citep{chen2020wavegrad, kong2020diffwave}, point cloud synthesis~\cite{luo2021diffusion, luo2021score, zhou20213d}, and many other AI tasks. 
These methods train a deep neural network to drive as drift force a diffusion process to generate data, and  
are shown to outperform competitors, mainly GANs and VAEs, on stability and sample  diversity~\citep{xiao2021tackling, ho2020denoising, song2020score}. %

However, 
a range of open challenges arise 
on understanding, analyzing, and improving diffusion-based models.  
On the conceptual and theoretical perspective, 
 existing methods %
have been derived from multiple angles,  
including denoising score matching \citep{vincent2011connection, song2019generative},
time reversed diffusion  \citep{song2020score}, 
and  variational bounds %
\citep{ho2020denoising},  
but these approaches leave many design choices whose relations and effects have been unclear and difficult to analyze. 
On the practical side, 
standard approaches tend to be slow in both training and inference due to the need of a large number of diffusion steps, and are restricted to generating 
continuous data in $\RR^d$ -- 
 special techniques such as 
 dequantization \citep{uria2013rnade, ho2019flow++}  and 
 multinomial diffusion~\citep{hoogeboom2021argmax, austin2021structured} need to be developed
 case by case for different types of discrete data 
 and the results still tend to be unsatisfying despite 
 promising recent advances \cite{hoogeboom2021argmax, austin2021structured}.

In this work, we approach diffusion models with a simple and classical statistical learning  framework. %
By viewing the diffusion models as 
a latent variable model 
consisting of unobserved trajectories whose end points output observed data, 
the learning is decomposed into two parts: 
1) constructing imputation mechanisms to generate latent trajectories that would have generated a given data point $x$, 
and 2) specifying and training the diffusion generative model to generate  data on the domain $\Omega$ of interest by maximizing likelihood using the imputed trajectories. 
Both components involve constructing \emph{diffusion bridge} processes, called $x$-bridge and $\Omega$-bridge, whose end points guarantee to hit a deterministic value $x$ or domain $\Omega$ at the terminal time, respectively. 
The design of learning algorithms reduces to constructing two bridges, for both which 
we provide a systematic study and a full suit of techniques. %
Our framework
allows us to decouple 
the various building blocks of the diffusion learning, enabling new theoretical analysis, algorithmic extensions to structured domains, and speedup in the regime of small sampling steps.  
Among others, we want to highlight two particular contributions: 

1) We develop  
a first error analysis 
for learning diffusion models including both statistical errors and time-discretization errors. In regime of classical asymptotic statistics \citep{van2000asymptotic}, we show that the KL divergence between the true and learned distributions from a variant of our method has an asymptotic rate of $\bigO{(\log(1/\epsilon)+1)/n+\epsilon}$, where $n$ is the number of i.i.d. data points and $\epsilon$ the step size in Euler discretization of the SDEs. 

2) Our framework 
is instantiated to 
provide a simple and universal approach to learning on data from an arbitrary domain $\Omega$ 
that can be embedded in $\RR^d$ and on which the expectation of truncated standard Gaussian distribution can be evaluated. 
This includes  product spaces of any type, bounded/unbounded, continuous/discrete, categorical/ordinal data, and their mix.
The efficiency of the method is testified on a suit of examples, including generating images, segmentation maps, %
and grid-valued point clouds.

\section{Learning Latent Diffusion Models}
\paragraph{Diffusion Generative Models} 
 \begin{wrapfigure}{r}{0.25\textwidth}
  \begin{center}
    \includegraphics[width=0.25\textwidth]{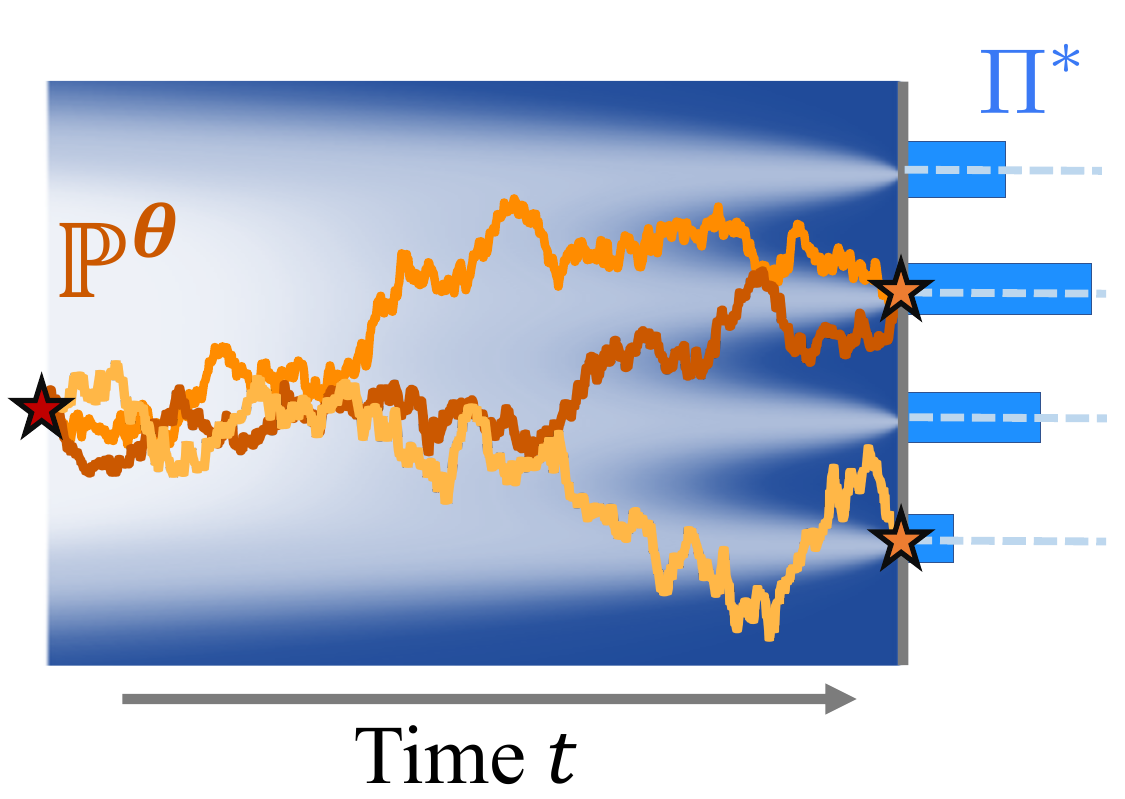}
  \end{center}
  \caption{\small{$\Omega$-Bridges for discrete $\Omega=\{1,2,3,4\}$.}} 
\end{wrapfigure}
 Let $\{x\datai\}_{i=1}^n$ be  
 an i.i.d. sample from an unknown distribution $\tg$ on a domain 
 $\Omega\subseteq \RR^d$. 
 We want 
 to fit the data with a diffusion model $%
 \P^\theta(\d Z)$, 
 which specifies the distribution of 
 a latent trajectory 
 $\rd Z = \{Z_{t} \colon t\in[0, \T]\}$
 that outputs an observation ($x= Z_\t$) at the terminal time $\t$. 
 The evolution of $Z$ is governed by an Ito process: 
 \bbb \label{equ:sthetax}
 \d \Z_t = s^\theta_t(\Z_t) \dt  + \sigma_t(\Z_t) \d W_t,~~ \forall t\in[0,\t],~~~~~~~ Z_0 \sim {\P_0^\theta},  
 \eee 
 where {$W_t$ is a Wiener process}; 
 $\sigma\colon [0,\t]\times \RR^d\to\RR^{d\times d}$ is a fixed, positive definite diffusion coefficient; 
  the drift term $s^\theta \in[0,\t] \times \RR^{d}\to \RR^{d}$ depends on a trainable parameter $\theta$ and is often specified  using a deep neural network. The initial distribution $\P_0^\theta$ is often a fixed elementary distribution (Gaussian or deterministic), but 
  we keep it trainable in the general framework. 
Here, $\P^\theta$ is the path measure on continuous trajectories $Z$ following \eqref{equ:sthetax}. We denote by $\P_t^\theta$ the marginal distribution of $Z_t$ at time $t$. 
We want to estimate $\theta$ such that the terminal distribution $Z_\T \sim \P_\T^\theta$ matches the data  $X\sim \tg$. 

If $\Omega$ is a strict subset of $\RR^d$,  e.g., bounded or discrete, 
then we need to specify the model $\P^\theta$  in \eqref{equ:sthetax} such that $\Z_\t$ is guaranteed to arrive $\Omega$ at $t=\t$ (while the non-terminal states may not belong to $\Omega$). 
 \begin{msg}  
 A process $Z$ in $\RR^d$ with law $\P$ is called a bridge to a set $B \subset \RR^d$, or $B$-bridge, 
 if $\P(Z_\T\in B)=1$. %
 \end{msg} 
Section~\ref{sec:constrained} discusses how to specify $s^\theta$ such that $\P^\theta$ is an $\Omega$-bridge. We assume $\Omega=\RR^d$ for now.

 \paragraph{A Poor man's EM}
 \begin{wrapfigure}{r}{0.3\textwidth}
  \begin{center}
    \includegraphics[width=0.3\textwidth]{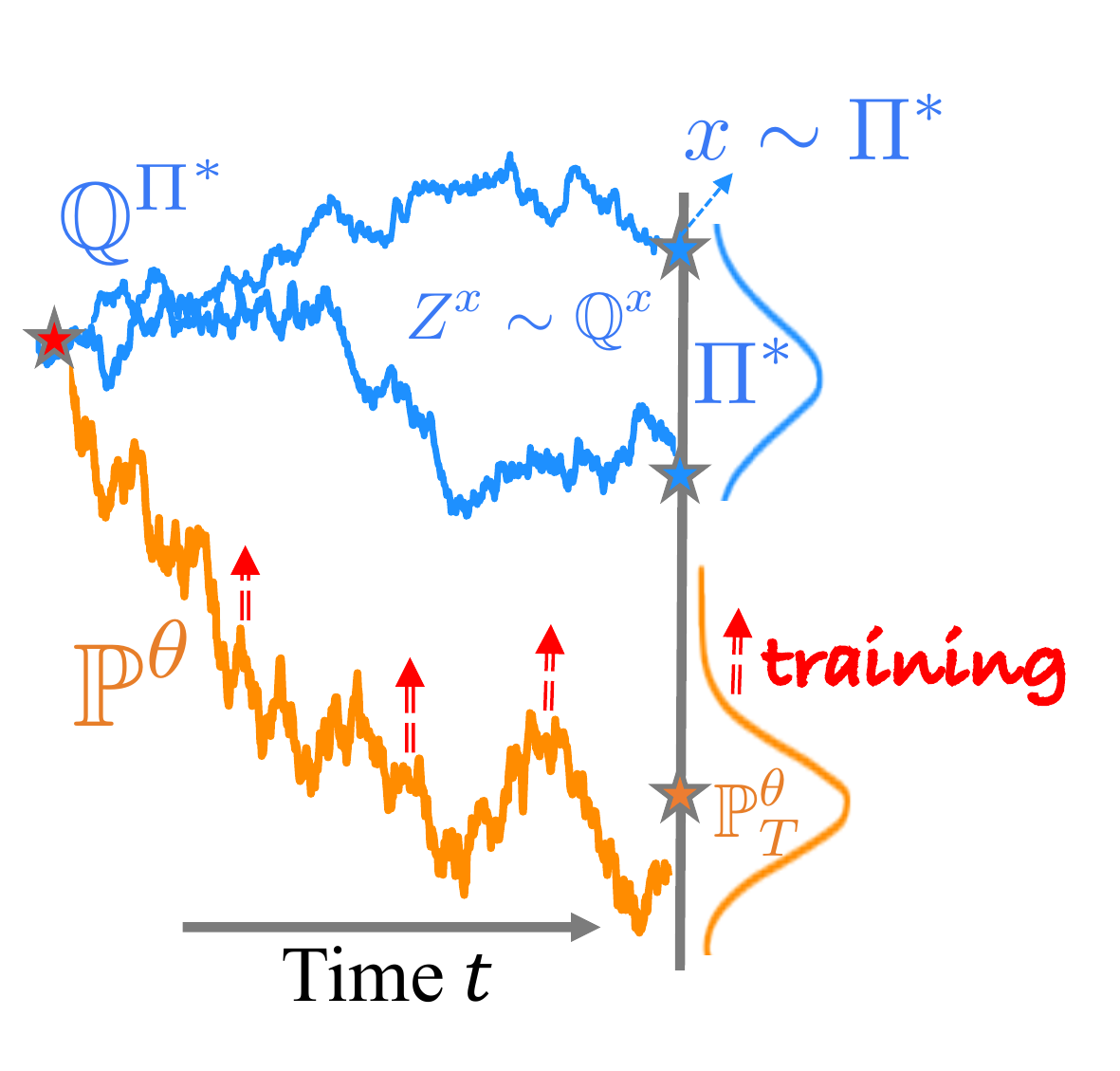}
  \end{center}
\end{wrapfigure}
  A canonical approach to learning latent variable models like $\P^\theta$
 is (variational) expectation maximization (EM), 
 which alternates between
 1) estimating the 
 posterior distribution $\P^{\theta,x}\defeq \P^\theta(Z|\Z_\t=x)$ of the latent 
 trajectories $Z$ given the observation (E-step), and 
 2) 
 estimating the model parameter $\theta$ with $Z$ imputed from $\P^{\theta,x}$ (M-step). 
Following  DDPM \citep{ho2020denoising}, 
we consider a simpler approach consisting of only the M-step, estimating $\theta$ with $Z$ drawn from $\Q^x\defeq \Q(Z|\Z_\t= x)$ of a pre-specified simple baseline process $\Q$, rather than the more expensive $\P^{\theta,x}$. 
Here  for each $x\in \Omega$, the conditioned process 
 $\Q^x\defeq \Q(Z|Z_\t = x)$  
 is the distribution of the trajectories from $\Q$ that are pinned at $x$ at time $t$.
 Therefore, $\Q^x$ is an $x$-bridge by definition. 
 
 Let $\Q^\tg(\cdot) = \int \Q^x(\cdot ) \tg(\dx)$ be the distribution of trajectories $Z$ generated in the following ``backward'' way: first drawing a data point $x\sim \tg$, and then $Z \sim \Q^x$ conditioned on the end point $x$. 
 This construction ensures that the terminal distribution of $\Q^\tg$ equals $\tg$, that is, $\Q^\tg_\t = \tg$.  %
 Then, the model $\P^\theta$ can be  estimated by fitting data drawn from $\Q^\tg$ using maximum likelihood estimator: 
\bbb \label{equ:kl0}
 \min_{\theta} \left \{ \L(\theta)\defeq  \KL(\Q^\tg ~||~ \P^\theta) 
\right\}. 
 \eee

The classical (variational) EM would alternatively update $\theta$ (M-step) and $\Q^x$ (E-step) to make $\Q^x \approx \P^{\theta,x}$. 
Why is it ok to simply drop the E-step? 
At the high level, it is the benefit from using universal approximators like deep neural networks: if the model space of $\P^\theta$ is sufficiently rich,  by minimizing the KL divergence in \eqref{equ:kl0}, $\P^\theta$ can approximate the given $\Q^\tg$ %
well enough (in a way that is made precise in sequel) %
such that their terminal distributions are close: $\P^\theta_\T \approx \Q^\tg_\T =  \tg$. 

\begin{msg} 
{Learning latent variable models require no E-step if the model space is sufficiently rich.} 
\end{msg}

We should see that 
in this case the latent variables $Z$ in 
the learned model $\P^\theta$ is 
\emph{dictated} by the choice of 
the imputation distribution $\Q$ since we have $\P^{\theta,x}=\Q^x$ when the KL divergence in \eqref{equ:kl0}  is fully minimized to zero; EM also achieves $\P^{\theta,x}=\Q^x$ but has the imputation distribution $\Q^x$ determined by the model $\P^\theta$, not the other way.

\paragraph{Loss Function}
 Let us assume that  $\Q^x$ yields  a general non-Markov diffusion process of form 
 \bbb \label{equ:zxt}  
\d Z_t =  \eta^x(Z_{[0,t]}, t) \dt + \sigma(Z_t, t)\dW_t,~~~~~ Z_0 \sim \mu^x,
\eee 
where the drift $\eta^x$ and initial distribution $\mu^x$ depend on the end point $x$ and the diffusion coefficient $\sigma$ is the same as that of $\P^\theta$. Here $\eta^x$ can depend on the whole trajectory upto time $t$ and hence $\Q^x$ can be non-Markov.   %
$\Q^x$ is Markov if $\eta^x(Z_{[0,t]}, t) = \eta^x(Z_{t}, t).$ 
See Section~\ref{sec:bridges} for instances of $\eta^x, \mu^x$. %

Using  Girsanov theorem \citep[e.g.,][]{oksendal2013stochastic},   
with $\P^\theta$
in \eqref{equ:sthetax} and 
$\Q^x$ in \eqref{equ:zxt}, 
the KL divergence  in \eqref{equ:kl0}
can be reframed into a form of the score matching loss from  \cite{song2020score, song2021maximum}:  
 
 \bbb \label{equ:lossscorem}
\L(\theta) 
 = 
 \E_{\substack{x\sim \tg \\ Z\sim \Q^x}}
 \!\!\!\left [ 
  \underbrace{- \log p_0^\theta(\X_0)}_{\text{MLE of initial dist.}} + 
 \frac{1}{2} \int_0^\T \!\!\!\underbrace{ \norm{ \sigma^{-1}(\X_t, t)(s^\theta(Z_t, t) - \eta^x(Z_{[0,t]}, t))}^2 }_{\text{score matching}} 
\! \df t
\right ] +  \const,  
\eee

where %
$p_0^\theta$ is the probability density function (PDF) of the initial distribution $\P^\theta_0$. 
Therefore, $\L(\theta)$ is a sum of the negative log-likelihood of the initial distribution that encourages  $\P_0^\theta\approx \Q^\tg_0$, 
and a least squares loss between $s^\theta$ and $\eta^x$. 
See practical implementation in Section~\ref{sec:practical} and Algorithm~\ref{alg:learning}. 

\paragraph{\mrk}  
As $\P^\theta$ is Markov by the model assumption, 
it can not perfectly fit $\Q^\tg$ 
which is non-Markov in general.  
This is a substantial problem because
$\Q^\tg$ can be non-Markov \emph{even if} $\Q^x$ is Markov for all $x\in \Omega$ (see  Section~\ref{sec:markov}).   
In fact, using Doob's $h$-transform method \citep{doob1984classical}, 
$\Q^\tg$ can be shown to be the law of a diffusion process %
\bb 
\d Z_t = \eta^\tg(Z_{[0,t]}, t) \dt + \sigma(Z_t, t)\d W_t, && 
\eta^\tg(z_{[0,t]}, t) = \E_{Z\sim \Q^\tg}\left [\eta^{Z_\T}(z_{[0,t]}, t)~|~Z_{[0,t]} = z_{[0,t]} \right],  
\ee 
where $\eta^\tg$ is the expectation 
of $\eta^x$ when $x=Z_\t$ is drawn from $\Q$ conditioned on $Z_{[0,t]}$. 

We resolve this by observing that 
it is not necessary to match the whole path  measure 
($\P^\theta\approx \Q^\tg$) to match the terminal  
($\P_\t^\theta \approx  \Q_\t^\tg = \tg$). It is enough for $\P^\theta$ to be the best Markov approximation (a.k.a. \mrk) of $\Q^\tg$, 
which matches all (hence terminal) fixed-time marginals with $\Qt$: 
\bb \proj(\Qt, \mathcal M) \defeq \argmin_{\P\in \mathcal M} \KL(\Qt ~||~ \P),
&& 
\text{$\mathcal M$ = the set of all Markov processes on $[0,\t]$}.  
\ee

\begin{pro}\label{thm:markov}
The global optimum of $\L(\theta)$ in \eqref{equ:kl0} and \eqref{equ:lossscorem} is achieved by $\theta^*$ if  
\bbb \label{equ:global} 
s^\thetat(z, t) = \E_{Z\sim \Q^\tg}\left [\eta^{Z_\T}(Z_{[0,t]}, t)~|~Z_t = z \right ], && \mu^{\theta\true}(\d z_0) = \Q^\tg_0 = \E_{x\sim \tg} \left [{\Q}_0^x(\d z_0) \right ]. 
\eee  
In this case, $ \P^{\theta\true} = \proj(\Q^\tg, \mathcal M)$ is the {\mrk} of $\Qt$, 
with which it matches all time-marginals: $\P^{\theta\true}_t = \Q^\tg_t$ for all time $t\in[0,T]$. In addition, %
\bbb \label{equ:kldfdfdffd}
\KL(\tg ~||~ \P^\theta_\t) \leq 
\KL(\P^{\theta\true}~||~ \P^\theta) 
= \KL(\Q^\tg ~||~ \P^\theta) -
\KL(\Q^\tg ~||~ \P^{\theta\true})  
= \L(\theta) - \L(\theta\true). 
\eee  
\end{pro}
Note that $s^\thetat$ is a conditional expectation of $\eta^\tg$: %
$s^\thetat(z,t) = \E_{Z\sim \Q^\tg} [\eta^{\tg}(Z_{[0,t]},t)~|~Z_t = z].$ 
Theorem 1 of \cite{peluchetti2021non} 
gives a related result that 
the marginals of mixtures of Markov diffusion processes can be matched by another Markov diffusion process, but does not 
discuss the issue of {\mrk} nor connect to KL divergence. 
Theorem 1 of \cite{song2021maximum} is the special case of \eqref{equ:kldfdfdffd} when $\Q^\tg$ is Markov.

\begin{algorithm}[t] 
\caption{Learning diffusion generative models on domain $\Omega\subseteq \RR^d$} 
\begin{algorithmic}
\STATE \textbf{Input}: 
Given a dataset $\{x\datai\}$ drawn from distribution $\tg$  on domain $\Omega\subseteq \RR^d$, learn a diffusion model $\P^\theta$. 
\STATE 1) Take a baseline process $\Q$, 
find an $x$-bridge $\Q^x$ and set the model $\P^\theta$ to be  $\Omega$-bridges as in \eqref{equ:omegabridge}. 
\STATE 2) Estimate $\theta$ by minimizing $\hat \L(\theta)$ in  \eqref{equ:disc_training} with stochastic gradient descent. %
\STATE 3) Draw approximate sample from $\P^\theta$ following \eqref{equ:disc_inference}.  
\STATE 
\textbf{Remark} If we initialize $\Q^x$ from a common distribution $\mu^x=\mu$, we can set $\P_0^\theta = \mu$ and drop the initial log-likelihood term $\log p_0^\theta$ in \eqref{equ:disc_training}. 
\end{algorithmic}
\label{alg:learning}
\end{algorithm} 

\negaspace

\section{Let us Build Bridges}
\label{sec:bridges} 

\negaspace

We discuss how to build bridges, 
both $\Q^x$ as $x$-bridges and $\P^\theta$ as $\Omega$-bridges for constrained domains.  
We first derive $\Q^x$ as the conditioned process $\Q(\cdot|Z_\t=x)$ by using time reversal and $h$-transform %
(Section~\ref{sec:time}-\ref{sec:h-transform}), 
and then construct new bridges 
using mixtures of existing bridges which allows us to decouple the choice of initialization and dynamics in bridges (Section~\ref{sec:mixing})  and clarify the Markov property of the resulting $\Q^\tg$ (Section~\ref{sec:markov}). 
In Section~\ref{sec:constrained}, 
we provide a general approach for constructing $\P^\theta$ as $\Omega$-bridges for constrained domains.

\negaspace
\subsection{Bridge Construction: Time Reversal}
\label{sec:time} 
\negaspace
SMLD and DDPM can be viewed as specifying $\barQ$ via its time-reversed process that starts at time $\t$ and proceed backwards to $t=0$. 
The conditioning on $Z_\t=x$ 
can be achieved by simply initializing the reversed process from $x$. 
Specifically, $\Q^x$ is defined as the law of $Z_t^x$ whose reversed process $\rev\X_{t}^x \defeq \X_{\T-t}^x$ follows a Markov diffusion process that starts at $\rev \Z_0^x=x$: %
\bbb \label{equ:dervZdd}
\d \rev\X_t^x = \rev \eta (\rev \X_t^x, \t-t) \d t +   \sigma(\rev \X_t^x, \t-t) \d  \tilde W_t,~~~~~~~~~ \rev\X_0^x = x,
\eee 
where $\rev W$ is a standard Brownian motion. 
Using the time reversion formula \citep[e.g.,][]{anderson1982reverse},  
$Z_t^x$ follows 
\bbb \label{equ:reverse}
\d \X_t^x = \left ( - \rev \eta (\X_t^x, t)  +  \frac{\dd_z (\sigma^2( Z^x_t,  t) \q^x_{t}(  Z^x_t ))}{ q^x_{t}(  Z^x_t )} \right )\d t +   \sigma( \X_t^x,t) \d   W_t,~~~~~~~~~ \X_0^x \sim  \Q^x_{0}, 
\eee 
 where $\sigma^2 = \sigma\sigma\tt$ denotes matrix square, and 
$\Q_t^x$ and $q_t^x$ are the distribution and density function of $Z^x_t = \rev Z^x_{\t-t}$ following \eqref{equ:dervZdd},
which formula needs to be derived. 

As summarized in \cite{song2020score},  
most existing works specify \eqref{equ:dervZdd} as an Ornstein–Uhlenbeck (O-U) process of form $\d \rev Z_t = -\alpha_{\t-t} \rev Z_t \dt + \varsigma_{\t-t} \d \rev W_t$. In particular,  
SMLD \citep{song2019generative, song2020improved}  uses $\alpha_t = 0$ (Variance Exploding (VE) SDE) and DDPM uses $\alpha_t = \varsigma_t^2/2$  (Variance Preserving (VP) SDE). 

\begin{exa}\label{exa:smld}
SMLD \citep{song2019generative,song2020improved} uses %
$\d\rev Z^x_t = \varsigma_{\T-t} \d \rev W_t$. 
Let  $\beta_t = \int_0^t\varsigma_s^2 \d s$. 
Then $Z^x_t = \rev Z^x_{\T-t}$ follows
\bbb \label{equ:bmsigma}
\df \X^x_t= \eta_{\mathrm{bb},\varsigma}^x(Z^x_t, t)\dt + \varsigma_t
\df W_t, &&\text{ with~~~~ $ \eta_{\mathrm{bb},\varsigma}^x(Z^x_t, t)= \varsigma_t^2  \frac{x - \X^x_t}{\beta_\T -\beta_t}$ and  $Z^x_0 \sim \normal(x, \beta_\T)$,}
\eee 
which is a Brownian bridge (BB) process. 
A simple case is when $\varsigma_t=1$ and $\eta_{\mathrm{bb},1}^x(z,t) = \frac{x - z}{\t-t}$. \\
Because the initial distribution $\Q^x_0 = \normal(x, \beta_\t)$ depends on data $x$, the initial distribution $\P_0^\theta$ of $\P^\theta$ should in principle be learned to fit the mixture $\int \Q^x_0\tg(\dx)$. 
But as suggested in SMLD,  we can set $\P_0^\theta = \normal(0, \beta_\T)$ as an approximation when $\beta_\t$ is very large compared to the variance of the data $\tg$. 
\end{exa}

\negaspace
\subsection{Bridge Construction:  $h$-transform}
\label{sec:h-transform} 
\negaspace
The conditioned process can be derived directly without resorting to time reversal \cite{peluchetti2021non}. 
Assume $\barQ $ follows $\d Z_t = b(Z_t, t) \dt + \sigma(Z_t, t) \d W_t$. Then by using Doob’s method of $h$-transforms \citep{oksendal2013stochastic}, 
the conditioned process 
$\Q^x (\cdot)\defeq \barQ(\cdot ~|~ Z_\T = x)$, if it exists,  can be shown to be the law of %
\bbb \label{equ:Xtrw} 
\df \X^x_t 
= \left ( b(\X^x_t, t) + {\sigma^2(
\X^x_t, t) \dd_{z} \log 
q_{\T|t}(x~|~\X^x_t) } \right ) \dt + \sigma(\X_t^x,t) \df W_t,~~~~ \X_0 \sim 
\barQ_{0|T}(\cdot ~|~ x), 
 \eee  
 where %
 $q_{\T|t}(x|z)$ is the density function of the transition probability $\barQ_{\T|t}(\dx | z) = \barQ(\Z_\T \in \dx |\Z_t = z)$, assuming it exists. 
The additional drift term $\sigma^2 \dd \log q_{\T|t}(x |z)$ 
plays the role of 
steering $\Z_t$ towards the target $\Z_\T = x$. %
The initial distribution can be calculated by Bayes rule: $\barQ_{0|T}(\d z | x)  \propto \barQ_0(\d z )q_{\t|0} (x| z)$. 
We should note that the drift term in \eqref{equ:Xtrw} is independent of the initialization $\barQ_0$, which allows us to decouple in Section~\ref{sec:mixing} the choices of initialization and drift in bridges.

 \begin{exa} %
If $\barQ$ is the law of $\d \X_t = \varsigma_t \d W_t$,  
we have $\barQ_{\T|t}(\cdot | z) 
= \normal(z, \beta_\T-\beta_t)$, where $\beta_t = \int_0^t\varsigma_s^2 \d s$. Hence $\Q^x = \barQ(\cdot|Z_\T = x)$ is the law of  
\bbb \label{equ:bmsigma2}
\df \X_t^x= \eta_{\mathrm{bb},\varsigma}^x(\Z^x_t, t)\dt + \varsigma_t
\df W_t, &&\text{with~~~~ $Z^x_0\sim \Q^x_0 = 
\barQ_{0|T}(\cdot | x)$,} 
\eee 
where $\Q_0^x(\d z) \propto \barQ_0(\d z) \phi(x ~|~z, \beta_\t-\beta_t)$, and $\phi(\cdot|\mu,\sigma^2)$ is the density function of $\normal(\mu,\sigma^2).$ 
 \end{exa} 

 The $\Q^x$ in \eqref{equ:bmsigma2} shares the same drift $\eta^{x}_{\mathrm{bb},\varsigma}$ as that of SMLD in \eqref{equ:bmsigma}, but has a different initialization that depends on $\barQ_0$.  Two extreme choices of $\barQ_0$ stand out: 
 
 1) The SMLD initialization can be viewed as the case when we initialize $\barQ$ with an improper ``uniform'' prior $\barQ_0 = 1$, corresponding $\barQ_0=\normal(0,v)$ with $v \to+\infty$.  DDPM can be similarly interpreted as 
taking $\barQ$ to be a forward time O-U process $\d  Z_t = \alpha_{t}  Z_t \dt + \varsigma_{t} \d  W_t$ with an improper uniform initialization $\barQ_0=1$ 
(see %
more discussion in Appendix). 
 
2) %
Let $z_0$ be any point that can reach $Z_\t = x$ under $\barQ$ in that $x \in \supp(\barQ_{\T|0}(\cdot | z))$. If we take $\barQ_0 = \delta_{z_0}$, the delta measure centered at $z_0$, the bridge $\Q^x$ has the same deterministic initialization $\Q^x_0 = \delta_{z_0}$. 
 Hence any deterministic initialization equipped with the drift in \eqref{equ:Xtrw} yields a conditional bridge. 
 This choice is particularly convenient because $\Q^x_0$ is independent of $x$, and hence $\P^\theta$ can be initialized at $\P_0^\theta = \delta_z$ without learning. %

\negaspace 
\subsection{Bridge Construction: Mixtures}  \label{sec:mixing}
It is an immediate observation that mixtures of bridges are bridges: Let  $\Q^{z, A}$ be a set of $A$-bridges indexed by a variable $z$, then $\Q^A \defeq \int \Q^{z,A} \mu(\d z)$ is an $x$-bridge for any distribution $\mu$ on $z$.  

A special case is to take the mixture of the conditional bridges in \eqref{equ:Xtrw} starting from different deterministic initialization, which 
shows that we can obtain a valid $x$-bridge by
equipping the same drift in \eqref{equ:Xtrw} with essentially \emph{any}  initialization. 
Hence, the choices of the drift force and initialization
in $\Q^x$ can be completely decouple, 
which is not  obvious  from the time reversal framework,  
since there different dynamics (e.g., VP-SDE, VE-SDE) have to designed to obtain different  $\Q_0^x$. 

\begin{pro} 
\label{thm:mup}
Let $\tilde \Q$ %
is a path measure and $\Omega_x$ is the set of $z$ for which $  \tilde \Q^{z_0,x}(\cdot) \defeq   \tilde\Q(\cdot | Z_\t = x, Z_0 = z_0)$ exists. Then $\Q^x \defeq  \int \tilde\Q^{z_0,x}  \mu(\d z_0~|~x)$ is an $x$-bridge, for any distribution $\mu$ on $\Omega\times \Omega$.  
\end{pro}

\subsection{Markov and Reciprocal Structures of $\Q^\tg$} \label{sec:markov} 
If $\Q^x$ is constructed as $\Q^x = \Q(\cdot|Z_\t=x)$, 
 it is easy to see
 that $\Q^\tg$ is Markov iff $\Q$ is Markov.  
If $\Q^x$ is constructed from mixtures of bridges as above, 
the resulting $\Q^\tg$ is more complex.   
In fact, simply varying the initialization $\mu$ in Proposition~\eqref{thm:mup} can change the Markov structure of $\Q^\tg$. 
\begin{pro}
\label{thm:takeqxtobe}
Take $\Q^x$ to be the dynamics in \eqref{equ:bmsigma2} initialized from $\Z_0 \sim \normal(0, v_0)$. Assume $\varsigma_t>0$, $\forall t\in[0,\t]$. Then $\Q^\tg$ is Markov only when  $v_0=0$, or $v_0 = +\infty$. 
\end{pro} 

The right characterization of  $\Q^\tg$  %
 from  
  Proposition~\eqref{thm:mup} involves 
 reciprocal processes \citep{leonard2014reciprocal}. %
\begin{mydef}\label{def:reciprocal} 
A process $Z$ with law $\Q$ on $[0,T]$ is said to be reciporcal if it can be written into $\Q = \int \tilde \Q^{z_0, z_\t}\mu(\d z_0, \d z_\T)$, where $\tilde \Q$ is a Markov process and $\tilde \Q^{z_0,z_\t} =\tilde \Q(\cdot |Z_0 = z_0, Z_\t = z_\t)$, 
and $\mu$ is a probability measure on $\Omega\times \Omega.$ %
\end{mydef}

\begin{pro} \label{thm:qtgis}
$\Q^\tg$ is reciprocal iff 
$\Q^x = \int \tilde \Q^{z_0, x} \mu(\d z_0~|~x) $ for a Markov $\tilde \Q$ and 
 distribution $\mu$. %
\end{pro}

Intuitively, a reciprocal process can be viewed 
as connecting the head and tail of a Markov chain, yielding a single loop structure. 
A characteristic property is  $\Q(X_{[s,t]}\in A  ~|~ \X_{[0,s]}, \X_{[t,T]}) = \Q(X_{[s,t]}\in A ~|~ \X_s, \X_t)$, where 
$A$ is any event that occur between time $s$ and $t$. Solutions of the  Schrodinger bridge problems are reciprocal processes \citep{leonard2014reciprocal}.

\subsection{Constructing $\Omega$-Bridges for  Constrained Domains}  
\label{sec:constrained}

If $\Omega$ is a constrained domain,   
we need to specify the model $\P^\theta$ such that it is an $\Omega$-bridge for any $\theta$.  
We provide a simple method that works for 
\emph{any domain on which integration of standard Gaussian density function can be calculated.}   
An importance class is product spaces of form $\Omega = I_1 \times I_2 \times \cdots  I_d$, where $I_i$ can be discrete sets or intervals in $\RR$.  

Our method consists of two steps: 
1) we first get a baseline $\Omega$-bridge by deriving the conditioned process $\barQ(\cdot ~|~ Z_\t \in \Omega)$ from  $\Q$ which by definition is an $\Omega$-bridge;
2) we then show that add extra drifts on top of it keeps the $\Omega$-bridge property unchanged under some minor regularity condition. %

In the first step, for any $\barQ$ following $\d Z_t = b(Z_t, t)\dt + \sigma(Z_t, t) \d W_t$, 
the $h$-transform method shows that 
the conditioned process 
$\Q^\Omega  \defeq \barQ(\cdot ~|~ Z_\t \in \Omega)$ follows $\d Z_t = \eta^\Omega(Z_t,t) \dt + \sigma(\Z_t, t)\d W_t$ with 
\bb 
\eta^\Omega(z, t) 
=  b(z,t) + \sigma^2(z,t)
\E_{x\sim \barQ_{\T|t,z, \Omega}}[\dd_z \log q_{\T|t}(x~|~z)], &&  \X_0 \sim  \Q_{0|T}(\cdot ~|~ X_T\in \Omega).  
 \ee  
Its drift term is similar to that of the $x$-bridge in \eqref{equ:Xtrw}, except that $x$ is now randomly drawn from
an $\Omega$-truncated transition probability:  $\Q_{\T|t,z,\Omega}(\dx~|~z):= \Q(Z_\t = \dx~|~ \Z_t = z, Z_\t \in \Omega)$. 
As an example, assuming $\barQ $ follows $\d Z_t = \varsigma_t \d W_t$, we can show that $\Q^\Omega$ yields the following  $\Omega$-bridge: 
\bb 
\d Z_t = \eta_{\mathrm{bb}, \varsigma}^{\Omega}(Z_t,t)  \dt  + \varsigma_t\dW_t, 
&& 
\eta_{\mathrm{bb},\varsigma}^{\Omega}(z,t) 
= \varsigma_t^2 \E_{x\sim \normal_{\Omega}(z,\beta_\t-\beta_t)} \left [ \frac{x - Z_t}{\beta_\t - \beta_t} \right],
\ee 
where $\normal_{\Omega}(z,\beta_\t-\beta_t) = \mathrm{Law}(Z~|~Z\in \Omega)$ when $Z\sim \normal(\mu,\sigma)$, which is an $\Omega$-truncated Gaussian distribution. %
Hence, we can calculate $\eta_{\mathrm{bb},\varsigma}^{\Omega}$ once we can evaluate the expectation of $\normal_{\Omega}(z,\beta_\t-\beta_t)$. 
A general case is when $\Omega = I_1\times \cdots I_d$, for 
which the expectation reduces to one dimensional Gaussian integrals; %
see Appendix for details.  %

In the second step, given  an $\Omega$-bridge $\Q^\Omega$, 
we construct a parametric model $\P^\theta$ by 
adding a learnable neural network $f^\theta$ in the drift and (optionally) starting from a learnable initial distribution $\mu^\theta$: 
\bbb \label{equ:omegabridge}
\P^\theta \colon && \d Z_t = 
(\sigma(Z_t,t)f^\theta(Z_t,t)+ \eta^{\Omega}(Z_t, t)) \dt +  \sigma(Z_t, t) \d W_t, ~~~~ Z_0 \sim \P_0^\theta. 
\eee 
\begin{pro}\label{thm:foranyqomega}
For any $\Q^\Omega$ following $\d Z_t = \eta^\Omega(Z_t,t)\dt + \d W_t$ that is an $\Omega$-bridge, the  $\P^\theta$ in \eqref{equ:omegabridge} 
is also an $\Omega$-bridge if $\E_{Z\sim \barQ^\Omega}[\int_0^\t\norm{f^\theta(Z_t, t)}_2^2\dt ]<+\infty$ and $\KL(\Q^\Omega_0~||~\P^\theta_0) < +\infty$. 
\end{pro}
The condition on $f^\theta$ is very mild, 
and it is satisfied if $f^\theta$ is bounded, as is the case for most neural networks. 
Further, using the mixture of initialization argument in Section~\ref{sec:mixing}, 
we can set the initialization $\P_0^\theta$ to be any 
distribution supported on the set of points that can reach $\Omega$ following $\Q$ (precisely, points $z_0$ that satisfy $\Omega \cap \supp(\Q_\t(\cdot|Z_0 =z_0))\neq \emptyset$).

\section{Practical Algorithms and Error Analysis}\label{sec:practical}
In practice, %
we need to introduce empirical and numerical approximations 
in both training and inference phases. 
Denote by $\tau = \{\tau_i\}_{i=1}^{K+1}$ a grid of time points with $0 = \tau_1 < \tau_2 \ldots < \tau_{K+1}=\t$. 
During training, we minimize  an empirical and time-discretized surrogate of $\L(\theta)$ as follows 
\bbb \label{equ:disc_training} 
\hat \L(\theta) = 
\frac{1}{n}\sum_{i=1}^n
\ell(\theta; Z\datai, \tau\datai), && 
\ell(\theta; Z, \tau)
\defeq - \log p_0^\theta(Z_0) + 
\frac{1}{2K}\sum_{k=1}^K\Delta(\theta; Z, \tau_k), 
\eee  
where $\Delta(\theta; Z, t)\defeq \norm{ \sigma^{-1}(\X_t, t)(s^\theta(Z_t, t) -\eta^x(Z_{[0,t]}, t))}^2$, and 
$\{Z\datai\}$ is drawn from $\Q^\tg$, and $\tau\datai$ can be either a deterministic uniform grid of $[0,\t]$, i.e., $\tau\datai= \{i/K\}_{i=0}^{K}$, or drawn i.i.d. uniformly on $[0,\t]$ (see e.g.,\cite{song2020score, ho2020denoising}).  %
A subtle problem here is that the variance of 
 $\Delta(\theta; Z, t)$ grows to infinite  as $t\uparrow \T$. 
Hence, we should not include $\Delta(\theta; Z, \t)$ at the end point $\tau^{K+1} = \t$ into the sum in the loss $\ell(\theta,Z,\tau)$ to avoid variance exploding.  %

In the sampling phase, the continuous-time model $\P^\theta$
should be approximated numerically.  
A standard approach is  the Euler-Maruyama method, 
which simulates the trajectory on a time grid $\tau$ by 
\bbb \label{equ:disc_inference}
\hat \X_{\tau_{k+1}} = \hat \X_{\tau_k} + \epsilon_k  s^\theta(\hat \X_{\tau_k}, \tau_k) + \sqrt{\epsilon_k} \sigma(\hat \X_{\tau_k}, \tau_k)  \xi_{k},
&& \epsilon_k = \tau_{k+1} - \tau_k, && \xi_{k}\sim \normal(0, I_{d}),
\eee  
The final output is $\hat \X_\t$. %
The following result shows 
the KL divergence between  $\tg$ and the distribution of $\hat Z_\t$ can be bounded by the sum of the step size and the expected optimality gap $\E[\hat \L(\theta) - \hat\L(\theta\true)]$ of the time-discretized loss in \eqref{equ:disc_training}.

\begin{pro}%
\label{thm:disc}
Assume $\Omega=\RR^d$ and  $\sigma(z, t) = \sigma(t)$ is state-independent. 
Take the uniform time grid $\tau^{\mathrm{unif}} \defeq \{i\epsilon\}_{i=0}^K$ with step size $\epsilon=\t/K$ in %
the sampling step \eqref{equ:disc_inference}. 
Assume $\sigma(t)>c>0$, $\forall t$ and $\sigma(t)$ is piecewise constant w.r.t. time grid $\tau^{\mathrm{unif}}$. 
Let $\L_\epsilon(\theta) = \E_{Z\sim \Q^\tg}[\ell(\theta; Z, \tau^{\mathrm{unif}})].$ 
Let $\P^{\theta, \epsilon}_\T$ be the distribution of the resulting sample $\hat \X_\t$. 
Let  $\theta\true$ be an optimal parameter satisfying \eqref{equ:global}.  Assume 
$C_0\defeq \sup_{z,t}\left (\norm{s^\thetat(z,t)}^2/(1+\norm{z}^2),~ \trace(\sigma^2(z,t)),~  \E_{\P^{\theta^*}}[\norm{Z_0}^2]\right) <+\infty$, and %
$\norm{s^\thetat(z,t) - s^\thetat (z',t')}^2_2 \leq L\left ( 
\norm{z -z'}^2 + \abs{t - t'} \right )$ for $\forall z,z'\in \RR^d$ and $t,t'\in[0,\t]$. 
Then %
$$
\sqrt{\KL(\tg~||~\P^{\theta,\epsilon}_\t )} \leq 
\sqrt{\L_\epsilon(\theta) - \L_\epsilon(\theta\true)}  + \bigO{\sqrt{\epsilon}}. 
$$
\end{pro}
\negaspace
To provide a simple analysis of the statistical error, we assume that  $\hat \theta_n = \argmin_\theta \hat \L(\theta)$ 
is an asymptotically normal M-estimator of $\theta\true$ following classical asymptotic statistics  \cite{van2000asymptotic},
with which we can estimate the rate of the excess risk $\L_\epsilon(\hat \theta_n) - \L_\epsilon(\theta\true)$ and hence the KL divergence. 

\begin{pro}%
\label{thm:asymptotic}
Assume the  conditions in Proposition~\eqref{thm:disc}. 
Assume $\hat \theta_n = \argmin_{\theta} \hat \L_\epsilon(\theta)$ with $\hat \L_\epsilon(\theta)=\sum_{i=1}^n \ell(\theta; Z\datai, \tau^{\mathrm{unif}})/n$, $Z\datai\sim \Q^\tg$. 
Take $\Q^x$ to be the standard Brownian bridge $\d Z^x_t = \frac{x-Z^x_t}{\T-t}\dt + \d W_t$ with $Z_0\sim \normal(0,v_0)$ and $v_0 >0$. 
Assume 
$
\sqrt{n}(\hat \theta_n-\theta\true) \dto 
\normal(0, \Sigma_*)$ as $n\to+\infty$, where $\Sigma_*$ is the asymptotic covariance matrix of the M estimator $\hat \theta_n$. 
Assume $\L_\epsilon(\theta)$ is second order continuously differentiable and strongly convex at  $\thetat$. 
Assume $\tg$ has a finite covariance and admits a density function $\pi$ 
that satisfies %
$\sup_{t\in[0,\t]}\E_{\Q^\tg}\left [\norm{\dd_\theta s^{\thetat}(Z_t, t)}^2(1+\norm{\dd\log\pi(Z_T)}^2+\trace(\dd^2\log \pi(Z_T)))
\right ]<+\infty$.
We have %
\bbb \label{equ:eklpi}
\E\left[\sqrt{\KL(\tg ~||~ \P_\t^{\hat \theta_n, \epsilon})}~ \right ] = \bigO{ \sqrt{\frac{\log (1/\epsilon)+1}{n}} + \sqrt{\epsilon} }.  
\eee 
\end{pro}
The expectation in Eq.~\ref{equ:eklpi} is w.r.t. the randomness of $\hat\theta_n$. 
The $\log (1/\epsilon)$ factor shows up as the sum of a harmonic series as the variance  of $\Delta(\theta; Z, t)$ grows 
with $\bigO{1/(\t-t)}$ when $t\uparrow \t$. %
Taking $\epsilon = 1/n$ yields  ${\KL(\tg ~||~ \P_\t^{\hat \theta_n, \epsilon})} = \bigO{\log n/n}$. 
If we want to achieve $\KL(\tg ~||~ \P_\t^{\hat \theta_n, \epsilon}) = \bigO{\eta}$, it is sufficient to take $K=\t/\epsilon = \bigO{1/\eta}$ steps 
and $n = \O(\log(1/\eta)/\eta)$ data points.

\section{Related Works}
Given that we pursuit an re-examination of a now popular framework,  
it is not surprising to share common findings with existing works. 
The time reversal method \citep{song2020score} amounts to a special approach to constructing bridges, but has the conceptual and practical disadvantage of entangling the choice of bridge dynamics and initialization. 
The $x$-bridge part of our framework overlaps  
with an independent work by \citet{peluchetti2021non}, which discusses a similar diffusion generative learning framework based on diffusion bridges and use $h$-transform to construct bridges in lieu of time reversal.  
We provide a complete picture of a more general framework 
based on maximum likelihood principle 
amendable to the first asymptotic error analysis,  %
with clarify of subtle issues of dropping E-step, initialization and  {\mrk}, extension to structured domains with $\Omega$-bridges.  
A number of works 
\citep[e.g.,][]{song2021maximum, huang2021variational} have approached diffusion models 
via the variational inference on stochastic processes, which is closely connected to our approach.  

A different approach is based on Schr\"odinger bridge~\citep{wang2021deep}, 
which, however, leads to more complicated algorithms that require iterative proportional fitting procedures. %
\cite{tzen2019theoretical} discusses 
sampling and inference in diffusion generative models through the lens of stochastic control, but does not touch learning.

\section{Experiments}
We evaluate our algorithms for generating 
 integer-valued point clouds, categorical semantic segmentation maps,  discrete and continuous CIFAR10 images. %
We observe that  1) our method provides a particularly attractive and superb approach to generating data from various discrete domains, 
2) our method shows significant advantages even on 
standard continuous data in terms of fast generation using very small number of diffusion steps. 

\begin{center}
\begin{minipage}{\textwidth}
\vspace{-10pt}
  \begin{minipage}{0.45\textwidth}
    \centering
    \includegraphics[width=6cm]{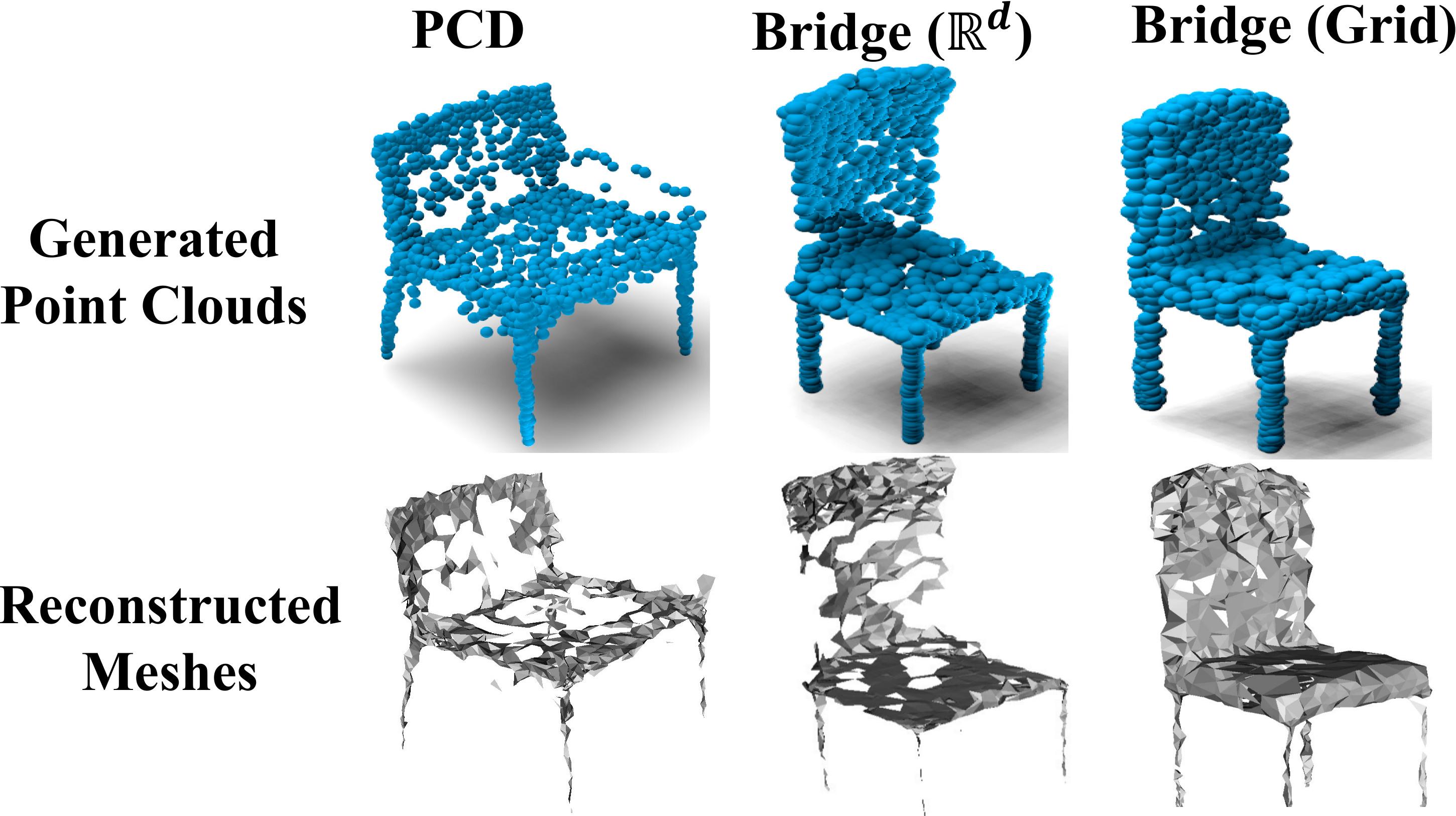}
  \end{minipage}
  \hfill
  \begin{minipage}{0.5\textwidth}
    \vspace{+8pt}
    \centering
    \renewcommand\arraystretch{1.8}
    \setlength{\tabcolsep}{1mm}
    \small{
    \begin{tabular}{c|ccc}
    \hline\hline
      Method & MMD $\downarrow$ & COV $\uparrow$ & 1-NNA $\downarrow$  \\ \hline
        PCD~\citep{luo2021diffusion} & 13.37 & 46.60 & 58.94\\
        Bridge ($\R^d$) & 13.30 & 46.52 & 59.32 \\
        Bridge (Grid) & \textbf{12.85} & \textbf{47.78} & \textbf{56.25} \\
        \hline \hline
      \end{tabular}}
      \vspace{5pt} 
    \captionsetup{labelformat=empty}
    \captionof{table}{}
    \label{tab:pc}
    \end{minipage}
  \end{minipage}
  \vspace{-10pt}
  \captionsetup{labelformat=andtable}
  \captionof{figure}{
  The point clouds (upper row) generated by different methods and meshes reconstructed from them (lower row).  
    Bridge (Grid) obtains more uniform points and hence better mesh thanks to the integer constraints.  
    Table shows the standard evaluation metrics for point cloud generation (numbers multiplied by $10^3$).} 
    \label{fig:pc}
\end{center}

\begin{figure}[h]
    \centering
    \vspace{-10pt}
    \includegraphics[width=1.0\textwidth]{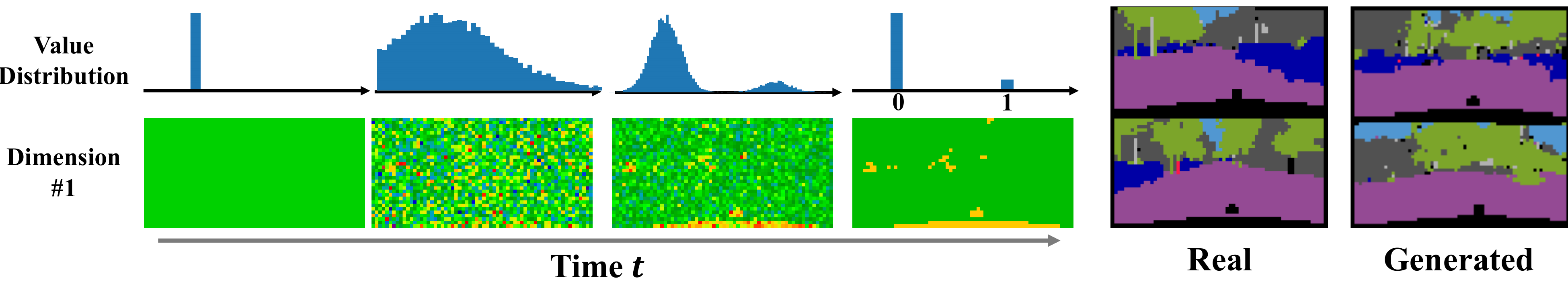}
    \caption{
    Results on generating categorical segmentation maps.
    Each pixel here an \texttt{one-hot} vector. 
    Each dimension of the $\Omega$-bridge starts from a deterministic and evolve through a stochastic trajectory to converge to either $0$ or $1$. 
    The generated samples have similar visual quality to the training data.
    }
    \label{fig:segmentation}
\end{figure}

\textbf{Generating Integer-valued Point Clouds}~~~ 
A feature of point clouds in 3D objects in graphics 
is that they tend to distribute even, 
especially if they are discretized from a mesh. 
This aspect is omitted in most existing works on point cloud generation. As a result they tend to generate non-uniform points 
that are unsuitable for real applications, which often involve converting back to meshes with procedures like Ball-Pivoting~\citep{bernardini1999ball}.   
We apply our method to generate point clouds that constrained on a integer grid which we show yields much more uniformly distributed points. 
To the best of our knowledge, we are the first work on integer-valued 3D point cloud generation.

A point cloud is a set of points  $\{x_i\}_{i=1}^m$, $x_i \in\RR^3$ in  the 3D space, where $m$ refers to the number of points. We apply two variants of our method: 
Bridge ($\R^d$) and Bridge (Grid). 
Both of the bridges use the process $\Q: \d  \Z_t = \d W_t$
starting at $\Z_0 = 0$, but on different domain $\Omega$.  
Bridge ($\R^d$) generates points in the continuous 3D space, i.e., $\Omega = \R^{3m}$. Bridge (Grid) generate points that on integer grids, $\Omega = \{1,\ldots,128\}^{3m}$. 
We test our method on 
ShapeNet~\citep{chang2015shapenet} chair models, and compare it with Point Cloud Diffusion  (PCD)~\citep{luo2021diffusion}, a state-of-the-art continuous diffusion-based  generative model for point clouds. The neural network $f^\theta$ in our methods are the same as that of PCD for fair comparison. %
Qualitative results and quantitative results are shown in Figure~\ref{fig:pc} and Table~\ref{tab:pc}.
As common practice~\citep{luo2021diffusion, luo2021score}, 
we measure minimum matching distance (MMD),  coverage score (COV) and 1-NN accuracy (1-NNA)  using Chamfer Distance (CD) with the test dataset.

\begin{minipage}{0.48\textwidth}
\centering
\renewcommand\arraystretch{1.05}
\setlength{\tabcolsep}{1mm}{
\scriptsize{
\resizebox{0.99\textwidth}{!}{
\begin{tabular}{l|cc}
    \hline \hline
    Methods & ELBO ($\downarrow$) & IWBO ($\downarrow$) \\ \hline
    Uniform Dequantization~\cite{uria2013rnade} & 1.010 & 0.930  \\ 
    Variational Dequantization~\cite{ho2019flow++} & 0.334 & 0.315 \\
    Argmax Flow (Softplus thres.)~\cite{hoogeboom2021argmax} & 0.303  & 0.290 \\
    Argmax Flow (Gumbel distr.)~\cite{hoogeboom2021argmax} & 0.365 & 0.341\\
    Argmax Flow (Gumbel thres.)~\cite{hoogeboom2021argmax} & 0.307 & 0.287 \\
    Multinomial Diffusion~\cite{hoogeboom2021argmax} & 0.305 & - \\ \hline
    Bridge-Cat. (Constant Noise) & 0.844 & 0.707 \\
    Bridge-Cat. (Noise Decay A) & \textbf{0.276} & \textbf{0.232} \\
    Bridge-Cat. (Noise Decay B) & 0.301  & 0.285 \\
    Bridge-Cat. (Noise Decay C) & 0.363   & 0.302 \\
    \hline \hline
\end{tabular}}}}
    \captionof{table}{Results on the CityScapes dataset.}
    \label{tab:segmentation}
\end{minipage}
\hfill
\begin{minipage}{0.48\textwidth}
    \centering
    \renewcommand\arraystretch{1.05}
    \setlength{\tabcolsep}{1mm}{
    \scriptsize{
    \resizebox{0.99\textwidth}{!}{
    \begin{tabular}{l|ccc}
    \hline \hline
    Methods & IS ($\uparrow$) & FID ($\downarrow$)  & NLL ($\downarrow$)  \\ \hline 
    \multicolumn{1}{l}{\textbf{Discrete}} & \multicolumn{1}{c}{} & & \\ \hline \hline
    D3PM uniform $L_{vb}$~\cite{austin2021structured} & 5.99 & 51.27 & 5.08 \\
    D3PM absorbing $L_{vb}$~\cite{austin2021structured} & 6.26 & 41.28 & 4.83 \\
    D3PM Gauss $L_{vb}$~\cite{austin2021structured} & 7.75 & 15.30 & 3.966 \\ 
    D3PM Gauss $L_{\lambda=0.001}$~\cite{austin2021structured} & 8.54 & 8.34 & 3.975 \\
    D3PM Gauss + logistic $L_{\lambda=0.001}$ & 8.56 & 7.34 & 3.435 \\\hline
    Bridge-Integer (Init. A) & \textbf{8.77} & \textbf{6.77} & 3.46 \\
    Bridge-Integer (Init. B) & 8.68 & 6.91 & \textbf{3.35} \\
    Bridge-Integer (Init. C) & 8.72 & 6.94  & 3.40 \\
    \hline \hline
    \end{tabular}}}}
    \captionof{table}{Discrete CIFAR10 Image Generation}
    \label{tab:discrete_cifar}
\end{minipage}

\textbf{Generating Semantic Segmentation Maps on CityScapes}~~~
\label{sec:segmentation}

\begin{wrapfigure}[24]{r}{0.5\textwidth}
    \includegraphics[width=0.5\textwidth]{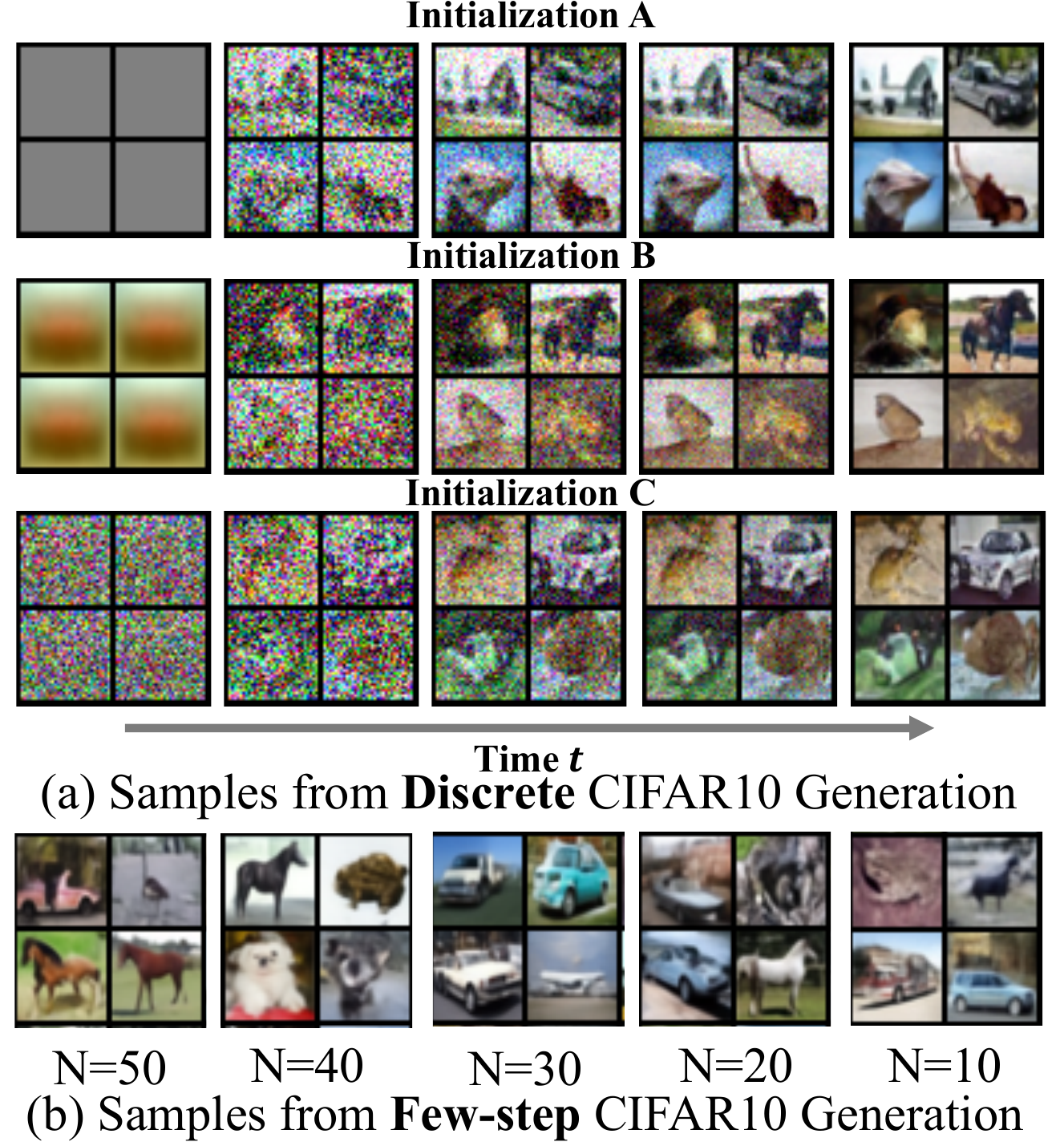}
  \caption{(a) Bridges can generate high-quality discrete samples with different initialization distribution. (b) The model can generate recognizable images even when the number of iterations is small.}
  \label{fig:cifar_imgs}
\end{wrapfigure}
We consider {unconditionally} generating categorical semantic segmentation maps. 
We represent each pixels as  a \texttt{one-hot} categorical vector.  
Hence 
the data domain is $\Omega = \{e_1,\ldots, e_c\}^{h\times w}$, where $c$ is the number of classes and $e_i$ is the $i$-th $c$-dimensional one-hot vector, and $h,w$ represent the height and width of the image. In CityScapes~\cite{Cordts2016Cityscapes}, $h=32, w=64, c=8$. 
We test a number of bridge models with $\Q: \d \Z_t = \varsigma_t \d W_t$ starting at the uniform point $Z_0 = 1/c$, with different schedule of the diffusion coefficient $\varsigma_t$, including   \emph{(Constant Noise)}: $\varsigma_t=1$;  \emph{(Noise Decay A)}: $\varsigma_t=a\exp(-bt)$;  \emph{(Noise Decay B)}: $\varsigma_t=a(1-t)$; \emph{(Noise Decay C)} $\varsigma_t = a - a\exp(-b(1-t))$. Here $a$ and $b$ are  hyper-parameters. 
We measure the negative log-likelihood (NLL) of the test set using the learned models. The NLL (bits-per-dimension) is estimated with evidence lower bound (ELBO) and importance weighted bound (IWBO)~\cite{burda2016importance}, respectively.
The results are shown in Figure~\ref{fig:segmentation} and Table~\ref{tab:segmentation}. 

\begin{table}[b]
    \centering
    \scriptsize{
    \begin{tabular}{l|c|ccccc}
    \hline \hline
    Methods & $K= 1000$ & $K=50$ & $K=40$ & $K=30$ & $K=20$ & $K= 10$ \\ \hline
    DDPM & 3.37 & 37.96 & 95.79 & 135.23 & 199.22 & 257.78 \\ 
    SMLD & \textbf{2.45} & 140.98 & 157.67 & 169.62 & 267.21 & 361.23 \\ \hline
    Bridge & 9.80 & 18.55 & 19.11 & 21.14 & 24.93 & 34.97 \\
    Bridge (Init. C) & 9.65 & \textbf{17.91} & \textbf{18.71} & \textbf{20.31} & \textbf{24.12} & \textbf{33.38} \\
    \hline \hline
    \end{tabular}}
    \caption{
    Results on continuous CIFAR10 generation when varying the number of diffusion steps in both training and testing. 
    Our method shows significant advantages in regime of small diffusion steps ($K\leq 50$). 
    }
    \label{tab:fewstep}
\end{table}

\textbf{Generating Discrete CIFAR10 Images}~~~
In this experiment, we apply three types of bridges. 
All of these bridges use the same output domain $\Omega=\{0,\ldots, 255\}^{h\times w\times c}$, %
where $h, w, c$ are the height, width and number of channels of the images, respectively. 
We set $\Q$ to be Brownian motion with the Noise Decay A in Section~\ref{sec:segmentation}, that is, $\Q : \d Z_t=\varsigma_t \d W_t$, where  $\varsigma_t=a \exp(-bt)$. We consider different initializations of $\Q$: 
\emph{(Init. A)} $Z_0 = {128}$;  \emph{(Init. B)} $Z_0 ={\hat \mu_0}$, \emph{(Init. C)} $Z_0 \sim \mathcal{N}(\hat \mu_0, \hat \sigma_0)$, where $\hat \mu_0$ and $\hat \sigma_0$ are the empirical mean and variance of pixels in the CIFAR10 training set. 
We compare with the variants of a state-of-the-art discrete diffusion model, D3PM~\citep{austin2021structured}. For fair comparison, we use 
the DDPM backbone~\citep{ho2020denoising} as the neural drift $f^\theta$ in our method, similar to D3PM. 
We report the Inception Score  (IS)~\cite{salimans2016improved}, Fréchet Inception Distance (FID)~\cite{heusel2017gans} and negative log-likelihood (NLL) of the test dataset. 
The results are shown in Table~\ref{tab:discrete_cifar} and Figure~\ref{fig:cifar_imgs}. 

\vspace{10pt}
\textbf{Generating Continuous CIFAR10 Images with Few-Step Diffusion Models}~~~
In this experiment, we consider training diffusion models with very few sampling steps to generate continuous CIFAR10 images. 
For bridge, we use $\Q: \Z_t = \d W_t$ initialized from $Z_0={0.5}$. %
For SMLD, we use the implementation of NCSN++ in~\citep{song2020score}. For DDPM, we use their original configuration. We use the DDPM backbone. 
We train the models with $K=10,20,30,40,50$ diffusion steps. Note that this is different from training with $K=1000$ steps, then sampling with fewer steps. Because in the latter case, the neural network is trained on more time steps which are unnecessary when sampling. This could hurt performance. The results are shown in Table~\ref{tab:fewstep} and Figure~\ref{fig:cifar_imgs}.

\section{Conclusion and Limitations}
We present a framework for learning diffusion generative models %
that enables both theoretical analysis and algorithmic extensions to structured data domains.  It leaves a number of directions for further explorations and improvement. 
For example, the practical impact of the choices of the bridges $\Q$, in terms of initialization, dynamics, and noise schedule, 
are still not well understood and need more systematical studies. 
The current error analysis works in the classical finite dimensional asymptotic regime and does not consider optimization error, 
one direction is to extend it to 
high dimensional and 
non-asymptotic analysis and consider the training dynamics with stochastic gradient descent equipped with neural network architectures, using techniques such as neural tangent kernels. 

\clearpage
\bibliography{reference}

\begin{thebibliography}{43}
\providecommand{\natexlab}[1]{#1}
\providecommand{\url}[1]{\texttt{#1}}
\expandafter\ifx\csname urlstyle\endcsname\relax
  \providecommand{\doi}[1]{doi: #1}\else
  \providecommand{\doi}{doi: \begingroup \urlstyle{rm}\Url}\fi

\bibitem[Anderson(1982)]{anderson1982reverse}
Brian~DO Anderson.
\newblock Reverse-time diffusion equation models.
\newblock \emph{Stochastic Processes and their Applications}, 12\penalty0
  (3):\penalty0 313--326, 1982.

\bibitem[Austin et~al.(2021)Austin, Johnson, Ho, Tarlow, and van~den
  Berg]{austin2021structured}
Jacob Austin, Daniel~D Johnson, Jonathan Ho, Daniel Tarlow, and Rianne van~den
  Berg.
\newblock Structured denoising diffusion models in discrete state-spaces.
\newblock \emph{Advances in Neural Information Processing Systems},
  34:\penalty0 17981--17993, 2021.

\bibitem[Bernardini et~al.(1999)Bernardini, Mittleman, Rushmeier, Silva, and
  Taubin]{bernardini1999ball}
Fausto Bernardini, Joshua Mittleman, Holly Rushmeier, Cl{\'a}udio Silva, and
  Gabriel Taubin.
\newblock The ball-pivoting algorithm for surface reconstruction.
\newblock \emph{IEEE transactions on visualization and computer graphics},
  5\penalty0 (4):\penalty0 349--359, 1999.

\bibitem[Brock et~al.(2018)Brock, Donahue, and Simonyan]{brock2018large}
Andrew Brock, Jeff Donahue, and Karen Simonyan.
\newblock Large scale gan training for high fidelity natural image synthesis.
\newblock \emph{arXiv preprint arXiv:1809.11096}, 2018.

\bibitem[Burda et~al.(2016)Burda, Grosse, and
  Salakhutdinov]{burda2016importance}
Yuri Burda, Roger~B Grosse, and Ruslan Salakhutdinov.
\newblock Importance weighted autoencoders.
\newblock In \emph{ICLR (Poster)}, 2016.

\bibitem[Chang et~al.(2015)Chang, Funkhouser, Guibas, Hanrahan, Huang, Li,
  Savarese, Savva, Song, Su, et~al.]{chang2015shapenet}
Angel~X Chang, Thomas Funkhouser, Leonidas Guibas, Pat Hanrahan, Qixing Huang,
  Zimo Li, Silvio Savarese, Manolis Savva, Shuran Song, Hao Su, et~al.
\newblock Shapenet: An information-rich 3d model repository.
\newblock \emph{arXiv preprint arXiv:1512.03012}, 2015.

\bibitem[Chen et~al.(2020)Chen, Zhang, Zen, Weiss, Norouzi, and
  Chan]{chen2020wavegrad}
Nanxin Chen, Yu~Zhang, Heiga Zen, Ron~J Weiss, Mohammad Norouzi, and William
  Chan.
\newblock Wavegrad: Estimating gradients for waveform generation.
\newblock In \emph{International Conference on Learning Representations}, 2020.

\bibitem[Cordts et~al.(2016)Cordts, Omran, Ramos, Rehfeld, Enzweiler, Benenson,
  Franke, Roth, and Schiele]{Cordts2016Cityscapes}
Marius Cordts, Mohamed Omran, Sebastian Ramos, Timo Rehfeld, Markus Enzweiler,
  Rodrigo Benenson, Uwe Franke, Stefan Roth, and Bernt Schiele.
\newblock The cityscapes dataset for semantic urban scene understanding.
\newblock In \emph{Proc. of the IEEE Conference on Computer Vision and Pattern
  Recognition (CVPR)}, 2016.

\bibitem[Dhariwal and Nichol(2021)]{dhariwal2021diffusion}
Prafulla Dhariwal and Alexander Nichol.
\newblock Diffusion models beat gans on image synthesis.
\newblock \emph{Advances in Neural Information Processing Systems}, 34, 2021.

\bibitem[Doob and Doob(1984)]{doob1984classical}
Joseph~L Doob and JI~Doob.
\newblock \emph{Classical potential theory and its probabilistic counterpart},
  volume 549.
\newblock Springer, 1984.

\bibitem[Du and Mordatch(2019)]{du2019implicit}
Yilun Du and Igor Mordatch.
\newblock Implicit generation and modeling with energy based models.
\newblock \emph{Advances in Neural Information Processing Systems}, 32, 2019.

\bibitem[Grathwohl et~al.(2019)Grathwohl, Wang, Jacobsen, Duvenaud, Norouzi,
  and Swersky]{grathwohl2019your}
Will Grathwohl, Kuan-Chieh Wang, J{\"o}rn-Henrik Jacobsen, David Duvenaud,
  Mohammad Norouzi, and Kevin Swersky.
\newblock Your classifier is secretly an energy based model and you should
  treat it like one.
\newblock \emph{arXiv preprint arXiv:1912.03263}, 2019.

\bibitem[Heusel et~al.(2017)Heusel, Ramsauer, Unterthiner, Nessler, and
  Hochreiter]{heusel2017gans}
Martin Heusel, Hubert Ramsauer, Thomas Unterthiner, Bernhard Nessler, and Sepp
  Hochreiter.
\newblock Gans trained by a two time-scale update rule converge to a local nash
  equilibrium.
\newblock \emph{Advances in neural information processing systems}, 30, 2017.

\bibitem[Ho et~al.(2019)Ho, Chen, Srinivas, Duan, and Abbeel]{ho2019flow++}
Jonathan Ho, Xi~Chen, Aravind Srinivas, Yan Duan, and Pieter Abbeel.
\newblock Flow++: Improving flow-based generative models with variational
  dequantization and architecture design.
\newblock In \emph{International Conference on Machine Learning}, pages
  2722--2730. PMLR, 2019.

\bibitem[Ho et~al.(2020)Ho, Jain, and Abbeel]{ho2020denoising}
Jonathan Ho, Ajay Jain, and Pieter Abbeel.
\newblock Denoising diffusion probabilistic models.
\newblock \emph{Advances in Neural Information Processing Systems},
  33:\penalty0 6840--6851, 2020.

\bibitem[Ho et~al.(2022)Ho, Saharia, Chan, Fleet, Norouzi, and
  Salimans]{ho2022cascaded}
Jonathan Ho, Chitwan Saharia, William Chan, David~J Fleet, Mohammad Norouzi,
  and Tim Salimans.
\newblock Cascaded diffusion models for high fidelity image generation.
\newblock \emph{Journal of Machine Learning Research}, 23\penalty0
  (47):\penalty0 1--33, 2022.

\bibitem[Hoogeboom et~al.(2021)Hoogeboom, Nielsen, Jaini, Forr{\'e}, and
  Welling]{hoogeboom2021argmax}
Emiel Hoogeboom, Didrik Nielsen, Priyank Jaini, Patrick Forr{\'e}, and Max
  Welling.
\newblock Argmax flows and multinomial diffusion: Learning categorical
  distributions.
\newblock \emph{Advances in Neural Information Processing Systems}, 34, 2021.

\bibitem[Huang et~al.(2021)Huang, Lim, and Courville]{huang2021variational}
Chin-Wei Huang, Jae~Hyun Lim, and Aaron~C Courville.
\newblock A variational perspective on diffusion-based generative models and
  score matching.
\newblock \emph{Advances in Neural Information Processing Systems},
  34:\penalty0 22863--22876, 2021.

\bibitem[Karras et~al.(2020)Karras, Aittala, Hellsten, Laine, Lehtinen, and
  Aila]{karras2020training}
Tero Karras, Miika Aittala, Janne Hellsten, Samuli Laine, Jaakko Lehtinen, and
  Timo Aila.
\newblock Training generative adversarial networks with limited data.
\newblock \emph{Advances in Neural Information Processing Systems},
  33:\penalty0 12104--12114, 2020.

\bibitem[Kong and Ping(2021)]{kong2021fast}
Zhifeng Kong and Wei Ping.
\newblock On fast sampling of diffusion probabilistic models.
\newblock In \emph{ICML Workshop on Invertible Neural Networks, Normalizing
  Flows, and Explicit Likelihood Models}, 2021.

\bibitem[Kong et~al.(2020)Kong, Ping, Huang, Zhao, and
  Catanzaro]{kong2020diffwave}
Zhifeng Kong, Wei Ping, Jiaji Huang, Kexin Zhao, and Bryan Catanzaro.
\newblock Diffwave: A versatile diffusion model for audio synthesis.
\newblock In \emph{International Conference on Learning Representations}, 2020.

\bibitem[Lejay(2018)]{lejay2018girsanov}
Antoine Lejay.
\newblock The girsanov theorem without (so much) stochastic analysis.
\newblock In \emph{S{\'e}minaire de Probabilit{\'e}s XLIX}, pages 329--361.
  Springer, 2018.

\bibitem[L{\'e}onard et~al.(2014)L{\'e}onard, R{\oe}lly, and
  Zambrini]{leonard2014reciprocal}
Christian L{\'e}onard, Sylvie R{\oe}lly, and Jean-Claude Zambrini.
\newblock Reciprocal processes. a measure-theoretical point of view.
\newblock \emph{Probability Surveys}, 11:\penalty0 237--269, 2014.

\bibitem[Liu et~al.(2021)Liu, Tong, and Liu]{liu2021sampling}
Xingchao Liu, Xin Tong, and Qiang Liu.
\newblock Sampling with trusthworthy constraints: A variational gradient
  framework.
\newblock \emph{Advances in Neural Information Processing Systems},
  34:\penalty0 23557--23568, 2021.

\bibitem[Luo and Hu(2021{\natexlab{a}})]{luo2021diffusion}
Shitong Luo and Wei Hu.
\newblock Diffusion probabilistic models for 3d point cloud generation.
\newblock In \emph{Proceedings of the IEEE/CVF Conference on Computer Vision
  and Pattern Recognition}, pages 2837--2845, 2021{\natexlab{a}}.

\bibitem[Luo and Hu(2021{\natexlab{b}})]{luo2021score}
Shitong Luo and Wei Hu.
\newblock Score-based point cloud denoising.
\newblock In \emph{Proceedings of the IEEE/CVF International Conference on
  Computer Vision}, pages 4583--4592, 2021{\natexlab{b}}.

\bibitem[Nichol and Dhariwal(2021)]{nichol2021improved}
Alexander~Quinn Nichol and Prafulla Dhariwal.
\newblock Improved denoising diffusion probabilistic models.
\newblock In \emph{International Conference on Machine Learning}, pages
  8162--8171. PMLR, 2021.

\bibitem[Oksendal(2013)]{oksendal2013stochastic}
Bernt Oksendal.
\newblock \emph{Stochastic differential equations: an introduction with
  applications}.
\newblock Springer Science \& Business Media, 2013.

\bibitem[Peluchetti(2021)]{peluchetti2021non}
Stefano Peluchetti.
\newblock Non-denoising forward-time diffusions.
\newblock 2021.

\bibitem[Ramesh et~al.(2022)Ramesh, Dhariwal, Nichol, Chu, and
  Chen]{ramesh2022hierarchical}
Aditya Ramesh, Prafulla Dhariwal, Alex Nichol, Casey Chu, and Mark Chen.
\newblock Hierarchical text-conditional image generation with clip latents.
\newblock \emph{arXiv preprint arXiv:2204.06125}, 2022.

\bibitem[Salimans et~al.(2016)Salimans, Goodfellow, Zaremba, Cheung, Radford,
  and Chen]{salimans2016improved}
Tim Salimans, Ian Goodfellow, Wojciech Zaremba, Vicki Cheung, Alec Radford, and
  Xi~Chen.
\newblock Improved techniques for training gans.
\newblock \emph{Advances in neural information processing systems}, 29, 2016.

\bibitem[Song et~al.(2020{\natexlab{a}})Song, Meng, and
  Ermon]{song2020denoising}
Jiaming Song, Chenlin Meng, and Stefano Ermon.
\newblock Denoising diffusion implicit models.
\newblock In \emph{International Conference on Learning Representations},
  2020{\natexlab{a}}.

\bibitem[Song and Ermon(2019)]{song2019generative}
Yang Song and Stefano Ermon.
\newblock Generative modeling by estimating gradients of the data distribution.
\newblock \emph{Advances in Neural Information Processing Systems}, 32, 2019.

\bibitem[Song and Ermon(2020)]{song2020improved}
Yang Song and Stefano Ermon.
\newblock Improved techniques for training score-based generative models.
\newblock \emph{Advances in neural information processing systems},
  33:\penalty0 12438--12448, 2020.

\bibitem[Song et~al.(2020{\natexlab{b}})Song, Sohl-Dickstein, Kingma, Kumar,
  Ermon, and Poole]{song2020score}
Yang Song, Jascha Sohl-Dickstein, Diederik~P Kingma, Abhishek Kumar, Stefano
  Ermon, and Ben Poole.
\newblock Score-based generative modeling through stochastic differential
  equations.
\newblock In \emph{International Conference on Learning Representations},
  2020{\natexlab{b}}.

\bibitem[Song et~al.(2021)Song, Durkan, Murray, and Ermon]{song2021maximum}
Yang Song, Conor Durkan, Iain Murray, and Stefano Ermon.
\newblock Maximum likelihood training of score-based diffusion models.
\newblock \emph{Advances in Neural Information Processing Systems}, 34, 2021.

\bibitem[Tzen and Raginsky(2019)]{tzen2019theoretical}
Belinda Tzen and Maxim Raginsky.
\newblock Theoretical guarantees for sampling and inference in generative
  models with latent diffusions.
\newblock In \emph{Conference on Learning Theory}, pages 3084--3114. PMLR,
  2019.

\bibitem[Uria et~al.(2013)Uria, Murray, and Larochelle]{uria2013rnade}
Benigno Uria, Iain Murray, and Hugo Larochelle.
\newblock Rnade: The real-valued neural autoregressive density-estimator.
\newblock \emph{Advances in Neural Information Processing Systems}, 26, 2013.

\bibitem[Van~der Vaart(2000)]{van2000asymptotic}
Aad~W Van~der Vaart.
\newblock \emph{Asymptotic statistics}, volume~3.
\newblock Cambridge university press, 2000.

\bibitem[Vincent(2011)]{vincent2011connection}
Pascal Vincent.
\newblock A connection between score matching and denoising autoencoders.
\newblock \emph{Neural computation}, 23\penalty0 (7):\penalty0 1661--1674,
  2011.

\bibitem[Wang et~al.(2021)Wang, Jiao, Xu, Wang, and Yang]{wang2021deep}
Gefei Wang, Yuling Jiao, Qian Xu, Yang Wang, and Can Yang.
\newblock Deep generative learning via schr{\"o}dinger bridge.
\newblock In \emph{International Conference on Machine Learning}, pages
  10794--10804. PMLR, 2021.

\bibitem[Xiao et~al.(2021)Xiao, Kreis, and Vahdat]{xiao2021tackling}
Zhisheng Xiao, Karsten Kreis, and Arash Vahdat.
\newblock Tackling the generative learning trilemma with denoising diffusion
  gans.
\newblock \emph{arXiv preprint arXiv:2112.07804}, 2021.

\bibitem[Zhou et~al.(2021)Zhou, Du, and Wu]{zhou20213d}
Linqi Zhou, Yilun Du, and Jiajun Wu.
\newblock 3d shape generation and completion through point-voxel diffusion.
\newblock In \emph{Proceedings of the IEEE/CVF International Conference on
  Computer Vision}, pages 5826--5835, 2021.

\end{thebibliography}

\clearpage

\onecolumn
\appendix
\section{Appendix} 

\subsection{Derivation of the main loss in Equation~(\ref{equ:lossscorem})} 
\begin{proof}[Proof of Equation~\eqref{equ:lossscorem}]
Denote by $\Q^x = \Q(\cdot | Z_\tau = x)$. 
Note that %
\bb 
\KL(\Q^\tg ~||~ \P^\theta) 
& = 
\E_{x \sim \tg, Z\sim \Q^x} \left [  \log \frac{\d\Qt}{\d \P^\theta}(Z)\right ]\\
& = 
\E_{x \sim \tg, Z\sim \Q^x} \left [  \log \frac{\d\Q^x}{\d \P^\theta}(Z) + \log \frac{\d \Qt}{\d \Q^x}(Z)\right ] \\
& = \E_{x\sim \Pi^*}\left [ \KL(\Q^x ~||~ \P^\theta) \right ] +\const, 
\ee 
where $\const$ denotes a constant that is independent of $\theta$. 
Recall that $\Q^x$ follows 
$\d Z_t = \eta^x(Z_{[0,t]}, t) \dt + \sigma(Z_t, t) \d W_t$,
and $\Pt$ follows 
$\d Z_t = s^\theta(Z_t, t) \dt + \sigma(Z_t, t) \d W_t$. %
By Girsanov theorem \citep[e.g.,][]{lejay2018girsanov}, 
\bb 
 \KL(\Q^x ~||~ \P^\theta)  
 & = \KL(\Q^x_0 ~||~ \P^\theta_0) 
 + \frac{1}{2}  \E_{Z\sim \Q^x}
 \left [
 \int_0^\t \norm{s^\theta(Z_t, t) - \eta^x(Z_{[0,t]}, t)}^2_2  \dt 
 \right ] \\ 
 & = 
   \E_{Z\sim \Q^x}\left [  
 -\log p_0^\theta(Z_0) + 
 \frac{1}{2} \int_0^\t \norm{s^\theta(Z_t, t) - \eta^x(Z_{[0,t]}, t)}^2_2  \dt 
 \right ]  + \const.  
\ee 
Hence 
\bb 
L(\theta)
& =  \E_{x\sim \tg, Z\sim \Q^x}\left [  
 -\log p_0^\theta(Z_0) +  \frac{1}{2} \int_0^\t \norm{s^\theta(Z_t, t) - \eta^x(Z_{[0,t]}, t)}^2_2  \dt 
 \right ]  + \const  \\ 
 & =  \E_{Z\sim \Qt}\left [  
 -\log p_0^\theta(Z_0) + 
  \frac{1}{2} \int_0^\t \norm{s^\theta(Z_t, t) - \eta^{Z_\t}(Z_{[0,t]}, t)}^2_2  \dt 
 \right ]  + \const . 
\ee 
\end{proof}

\subsection{Derivation of the drift $\eta^\tg$ of $\Qt$}
\begin{lem}
Let $\Q^x$ is the law of 
$$
\d Z^x_t =  \eta^x(Z^x_{[0,t]}, t) \dt + \sigma(Z^x_t, t) \dW_t, ~~~ \Z_0 \sim \Q^x_0, 
$$
and $\Q^\tg \defeq \int \Q^x(Z) \tg(\dx )$ for a distribution  $\tg$ on $\RR^d$. 
Then  $\Q^\tg$ is the law of  
$$
\d Z_t =  \eta^\tg(Z_{[0,t]}, t) \dt + \sigma(Z_t, t) \d W_t, ~~~~ \Z_0 \sim \Q_0^\tg,
$$
where 
\bb 
\eta^\tg 
(z_{[0,t]}, t) = \E_{x\sim \tg, Z\sim \Q^x}
[\eta^x(Z_{[0,t]}, t) ~|~ Z_{[0,t]} =z_{[0,t]}], &&
\Q^\tg_0(\d z_0) = \E_{x\sim \tg} [\Q^x_0(\d z_0 )]. 
\ee 
\end{lem}
\begin{proof}
$\Q^\tg$ is the solution of the following optimization problem: 
$$
\Q^\tg = \argmin_{\P}  
\left \{ \KL(\Q^\tg~||~ \P)  = 
\E_{x\sim \tg}[\KL(\Q^x ~||~ \P)] + \const \right\}.
$$
By Girsanov's Theorem \cite[e.g.,][]{lejay2018girsanov}, 
any stochastic process $\P$ that has $ \KL(\Q^x~||~ \P)<+\infty$
(and hence is equivalent to $\Q^x$) has a form of $\d Z_t =  \eta^\tg(Z_{[0,t]}, t) \dt + \sigma(Z_t, t) \d W_t$ for some measurable function $\eta^\tg$, and 
\bb
&\E_{x\sim \tg}[\KL(\Q^x ~||~ \P)]  \\
&=  \E_{x\sim \tg}[\KL(\Q^x_0 ~||~\P_0)] + 
\E_{x\sim \tg, Z\sim \Q^x}\left [\frac{1}{2}\int_{0}^\t 
\norm{\sigma(Z_t,t)^{-1}(\eta^\tg(Z_{[0,t]}, t) - \eta^x(Z_{[0,t]}, 0))}_2^2 \right ]. 
\ee 
It is clear that to achieve the minimum, we need to take $\P_0(\cdot) = \E_{x\sim \tg}[\Q_0^x(\cdot)]$ and 
$\eta^\tg (z_{[0,t]}, t) = \E_{x\sim \tg, Z\sim \Q^x}[\eta^x(Z_{[0,t]}, t) ~|~ Z_{[0,t]} =z_{[0,t]}]$, which yields the desirable form of $\Q^\tg$.  
\end{proof}

\subsection{Derivation of {\mrk} (Proposition~\ref{thm:markov})}

\begin{proof}[Proof of Proposition~\ref{thm:markov}]
It is the combined result of Lemma~\ref{lem:markovdd2} and Lemma~\ref{lem:mrk_22} below. 
\end{proof}

\begin{lem}\label{lem:markovdd2}
Let  $\Q$ be a non-Markov diffusion process on $[0,\t]$ of form 
\bb 
&\Q: ~~~~~\d Z_t = \eta(Z_{[0,t]}, t)\dt + \sigma(Z_t, t) \dW_t,~~~ Z_0 \sim \Q_0, 
\ee 
and $\meas M = \argmin_{\P \in \mathcal M} \KL(\Q ~||~\P)$ be the {\mrk} of $\Q$, where $\mathcal M$ is the set of all Markov processes on $[0,\T]$. Then $\Q$ is the law of 
\bb 
&\meas M: ~~~~~\d Z_t = m(Z_t, t)\dt + \sigma(Z_t, t) \dW_t,~~~ Z_0 \sim \Q_0, 
\ee 
where 
$$
m(z,t) = \E_{Z\sim \Q} [\eta(Z_{[0,t]}, t)~|~ Z_t = z]. 
$$
In addition, we have $\Q_t = \meas M_t$ for all time $t\in[0,\t]$. 

\end{lem}
\begin{proof}
By Girsanov's Theorem \citep[e.g.,][]{lejay2018girsanov}, 
any  process that has $ \KL(\Q ~||~ \meas M)<+\infty$
(and hence is equivalent to $\Q$) has a form of $\d Z_t =  m(Z_{[0,t]}, t) \dt + \sigma(Z_t, t) \d W_t$, where $m$ is a measurable function. Since $\meas M$ is Markov, we have $m(Z_{[0,t]}, t) = m(Z_t, t)$. 
Then 
\bb
\KL(\Q ~||~ \P) 
=  \KL(\Q_0 ~||~\P_0) + 
\E_{Z\sim \Q}\left [\frac{1}{2}\int_{0}^\t 
\norm{\sigma(Z_t, t)^{-1}(\eta(Z_{[0,t]}, t) - m(Z_{t}, 0))}_2^2 \right ]. 
\ee 
It is clear that to achieve the minimum, we need to take $\meas M_0 = \Q_0$ and 
$m(z, t) = \E_{Z\sim \Q}[\eta(Z_{[0,t]}, t) ~|~ Z_t=z]$.

To prove $\Q_t = \meas M_t$, note that by the chain rule of KL divergence: 
$$
\KL(\Q~||~\P) = 
\KL(\Q_t~||~ \P_t) +  \E_{Z_t\sim \Q_t}[\KL(\Q(\cdot|Z_t)~||~ \P(\cdot|Z_t))], ~~~\forall t \in[0,\t]. 
$$
As the second term $\P(\cdot | Z_t)$ is independent of the choice of the marginal 
$\P_t$ at time $t\in[0,\t]$, 
the optimum should be achieved by $\M$ only if $\M_t = \Q_t$. 
\end{proof}

\begin{lem}\label{lem:mrk_22}
Let 
\bb 
&\Q: ~~~~~\d Z_t = \eta(Z_{[0,t]}, t)\dt + \sigma(Z_t, t) \dW_t,~~~ Z_0 \sim \Q_0\\
& \meas M: ~~~~~\d Z_t = m(Z_t, t)\dt + \sigma(Z_t, t) \dW_t,~~~ Z_0 \sim \Q_0, \\ 
&  \P^\theta: ~~~~~\d Z_t = s^\theta(Z_t, t)\dt + \sigma(Z_t, t) \dW_t,~~~ Z_0 \sim \P^\theta_0, 
\ee 
where $\meas M$ is the {\mrk} of $\Q$ (see Lemma~\ref{lem:markovdd2}). Then 
$$
\KL(\Q~||~\P^\theta) =
\KL(\Q~||~\meas M)  + \KL(\meas M~||~\P^\theta). 
$$
Hence, assume there exists $\theta\true$ such that $\P^{\theta\true} = \M$ and write  $\L(\theta) \defeq  \KL(\Q~||~\P^\theta).$ We have 
\bb \KL(\meas Q_\t ~||~\P^\theta_\t) = 
\KL(\meas M_\t ~||~\P^\theta_\t)\leq 
\KL(\meas M ~||~ \P^\theta)= \L(\theta) - \L(\theta\true). 
\ee 
\end{lem}
\begin{proof}
Note that 
\bb
 & \KL(\meas M ~||~ \P^\theta)  \\
 & %
= 
\KL(\M_0 ~||~\P_0^\theta) + 
\frac{1}{2}
\E_{Z_t\sim \M_t}\left [\int_{0}^\t 
\norm{\sigma(Z_t, t)^{-1}(s^\theta(Z_t, t) - m(Z_{t}, t))}_2^2 \right ] \dt \\ 
& =  \KL(\meas M_0 ~||~\P_0^\theta) + 
\frac{1}{2}\int_{0}^\t 
\E_{Z_t\sim \meas M_t}\left [ \norm{\sigma(Z_t, t)^{-1}(s^\theta(Z_t, t) - m(Z_{t}, t))}_2^2 \right ] \dt 
\\ 
& = 
\KL(\Q_0 ~||~\P_0^\theta) + 
\frac{1}{2}\int_{0}^\t 
\E_{Z_t\sim \Q_t}\left [
\norm{\sigma(Z_t, t)^{-1}(s^\theta(Z_t, t) - m(Z_{t}, t))}_2^2 \right ] \dt \ant{$\Q_t = \M_t~~\forall t$}\\ 
& = \KL(\Q_0 ~||~\P_0^\theta) + 
\E_{Z\sim \Q}\left [ \frac{1}{2}\int_{0}^\t 
\norm{\sigma(Z_t, t)^{-1}(s^\theta(Z_t, t) - m(Z_{t}, t))}_2^2\dt  \right ] \\
& = \KL(\Q_0 ~||~\P_0^\theta) + 
\frac{1}{2}\norm{s^\theta - m}_{\Q,\sigma}^2,
\ee 
where we define 
$
\norm{f}_{\Q,\sigma}^2 = \E_{Z\sim \Q}\left [ \frac{1}{2}\int_{0}^\t 
\norm{\sigma(Z_t, t)^{-1}f(Z_t, t)}_2^2\dt  \right ]. 
$

On the other hand, 
\bb
\KL(\Q ~||~ \P^\theta)  
& =  \KL(\Q_0 ~||~\P_0^\theta) + 
\E_{Z\sim \Q}\left [\frac{1}{2}\int_{0}^\t 
\norm{\sigma(Z_t, t)^{-1}(s^\theta(Z_{t}, 0))-\eta(Z_{[0,t]}, t) }_2^2 \dt  \right ] \\ 
& = \KL(\Q_0 ~||~\P_0^\theta) + \frac{1}{2} \norm{s^\theta - \eta}_{\Q, \sigma}^2 \\ 
 \KL(\Q ~||~ \M) & 
= %
\E_{Z\sim \Q}\left [\frac{1}{2}\int_{0}^\t 
\norm{\sigma(Z_t, t)^{-1}(\eta(Z_{[0,t]}, t) - m(Z_{t}, 0))}_2^2 \dt \right ] \\
& =\frac{1}{2} \norm{\eta - m}_{\Q, \sigma}^2.
\ee 
Using  Lemma~\ref{lem:vb} with $a(z) = \sigma(z,t)^{-1} s^\theta(z,t)$, and $b(z_{[0,t]}) =  \sigma(z,t)^{-1} \eta(z_{[0,t]},t)$, 
we have the following bias-variance decomposition: $$ \norm{\eta - s^\theta}_{\Q, \sigma}^2 =  \norm{s^\theta - m}_{\Q, \sigma}^2 +  \norm{\eta - m}_{\Q, \sigma}^2.$$ 
Hence, $\KL(\Q ~||~ \P^\theta) =\KL(\M ~||~ \P^\theta) + \KL(\Q~||~\M)$. %

Finally, $\KL(M_\t~||~\P_\t^\theta) \leq \KL(\M~||~\P^\theta)$ is the direct result of the following factorization of KL divergence:
$$
\KL(\M~||~\P^\theta) = \KL(\M_\t~||~\P_\t^\theta)  + \E_{x\sim \M_\t} \left[ \KL(\M_\t(\cdot | Z_\t = x)~||~\P_\t^\theta(\cdot | Z_\t = x) ) \right]. 
$$
\end{proof}

\begin{lem}\label{lem:vb}
Let $(X,Y)$ be a random variable and $a(x)$, $b(x,y)$ are square integral functions. Let $m(x) = \E[b(X,Y)~|~X = x]$. We have 
$$
\E[\norm{a(X) - b(X, Y)}^2_2] = 
\E[\norm{a(X) - m(X)}^2_2] + \E[\norm{b(X,Y)-m(X)}^2_2].
$$
\end{lem}
\begin{proof}
\bb
\E[\norm{a(X) - b(X, Y)}^2_2] 
& = \E[\norm{a(X) - m(X) + m(X) -  b(X, Y)}^2_2]  \\
& = \E[\norm{a(X) - m(X)}^2_2] + \E[ \norm{m(X) -  b(X, Y)}^2_2 ] + 2\Delta,
\ee 
where 
\bb
\Delta 
& = \E[(a(X) -m(X))\tt (m(X) - b(X,Y))] ] \\
& = \E[(a(X) -m(X))\tt \E[(m(X) - b(X,Y))|X]] \\
& = \E[(a(X) -m(X))\tt(m(X) - m(X))] 
 = 0.
\ee 
\end{proof}

\subsection{SMLD and DDPM as bridges with uninformative initialization}

We show that methods like SMLD and DDPM that specify $\Q^x$ as a time reversed 
O-U process $Z^x_t = \rev Z^x_{\T-t}$ with $\d \rev Z^x_t = -\alpha_{\T-t} \rev Z^x_t \dt + \varsigma_{\t-t} \d\tilde W_t$ and $\rev Z_0 = x$ 
can be viewed as taking $\Q^x =  \Q(\cdot | Z_\t = x)$ 
with $\Q$ the law of $\d Z_t = \alpha_t Z_t \dt  + \varsigma_t \dt $ 
initialized from $\normal(0, v)$ with $v \to +\infty$. 
This is made concrete in the following result.

\begin{pro} \label{thm:revou}
Let 
$\Q^x$ be the law of $Z^x_t = \rev Z^x_{\T-t}$ following $\d \rev Z^x_t = -\alpha_{\T-t} \rev Z^x_t \dt + \varsigma_{\t-t} \d\tilde W_t$ with $\rev Z_0 = x$. Assume $\sup_{t\in[0,T]}\{\alpha_t, \varsigma_t\} < \infty$. Let $\Q_v := \ito(\{\alpha_t\}, \{\beta_t\}, \normal(0,v))$ be the law of $\d Z_t = \alpha_t Z_t \dt + \varsigma_t \d W_t$ starting from $Z_0\sim \normal(0,v)$, where $v$ is the variance of the initial distribution. Then we have $\Q^x = \lim_{v\to+\infty}  \Q_v(\cdot | Z_\T = x)$, where the limit denotes weak convergence.
\end{pro}  
\begin{proof}
Let $\Q_v$ be the law of  the O-U process $\d Z_t  = \alpha_t Z_t \d t + \varsigma_t \d W_t$ initialized at $Z_0 \sim \normal(\mu_0, v)$.  Its solution is 
\bb 
Z_t 
 = \mu_{t|0} Z_0 + \int_0^t  \mu_{t|s}  \varsigma_s \d B_s  \sim \normal(\alpha_{t|0} \mu_0, ~ \alpha_{t|0}^2 v  + \beta_{t|0}),
\ee 
where we define  
\bb 
\alpha_{t|s} = \exp\left (\int_s^t \alpha_r \d r \right),  &&
\beta_{t|s} = \int_{s}^t \alpha_{t|r}^2 \varsigma_r^2 \d r,&& \forall 0\leq s\leq t\leq T.
\ee 
Using the time reversal formula~\cite{anderson1982reverse}, 
 the time-reversed process $\rev Z_t = Z_{T-t}$ follows  
\bb 
\d \rev Z_t = \left ( - \alpha_{\T-t}\rev Z_t  + 
r(\rev Z_t,  t) \right)
\dt + \varsigma_{\T-t} \d \rev W_t, && 
r(\rev Z_t, t) = \varsigma_{\T-t}^2 \frac{\alpha_{\T-t|0} \mu_0 - \rev Z_t}{\alpha_{\t-t|0}^2 v + \beta_{\T-t|0}},
\ee 
where $\rev W_t$ is a standard Brownian motion.  %
Taking $v\to+\infty$, the extra drift term due to the time reversion is vanished, and hence we get the follow process in the limit: 
$$
\d \rev Z_t = - \alpha_{\T-t} \dt 
+ \varsigma_{\T-t} \d \rev W_t. 
$$
This is directly reverting $\d Z_t = \alpha_t \dt + \varsigma_t \d W_t$ without introducing the extra term in the time reversal formula.  
\end{proof}

\subsection{Markov and Reciprocal Properties of $\Q^\tg$} 

\begin{proof}[Proof of Proposition~\ref{thm:mup}]
This is an obvious result. 
We have $\Q^{z_0, x}(Z_\t = x) = 1$ by the definition of conditioned processes. 
Hence $\Q^x(Z_\t = x) = \int\Q^{z_0, x}(Z_\t = x) \mu(\d z_0 ~|~x) = \int \mu(\d z_0 ~|~x)  = 1$. 
\end{proof}

\begin{pro} \label{thm:assumeqx}
Assume $\Q^x =  \Q(\cdot ~|~Z_\t= x)$ and 
$\pi^*(z) \defeq \frac{\d \tg}{\d \Q_\T}(z)$ exists and is positive everywhere. Then 
 $\Q^\tg$ is Markov, iff $\Q$ is Markov. 
\end{pro} 
\begin{proof}
If $\Q^x = \Q(\cdot ~|~ Z_\t = x)$, we have from the definition of $\Qt$: 
$$
\Q^\tg(Z) =  \Q(Z | \Z_\T) \tg(\Z_\T) =  \Q(Z) \pi^*(Z_\T),  
$$
where $\pi^*(Z_\T) = \frac{\d \tg}{\d  \Q_\T}(\Z_\T).$
Therefore, $\Q^\tg$ is obtained by multiplying a positive factor $\pi^*(Z_\t)$ on the terminal state $\Z_\T$ of $\Q$. Hence $\Qt$ has the same Markov structure as that of $\Q$.  
\end{proof}

\begin{proof}[Proof of Proposition~\ref{thm:takeqxtobe}]
When taking $\Q^x$ to be the dynamics \eqref{equ:bmsigma2} initialized from $Z_0 \sim \mu_0 = \normal(0,v_0)$, we have $\Q^x = \int \mu_0(\d z_0) \tilde \Q^{z_0, x}$, where 
$\tilde\Q^{z_0, x} = \tilde\Q(\cdot | Z_0 = z_0, Z_\t = x)$ with $\tilde\Q$ following Brownian motion $\d Z_t = \d W_t$.  Hence, we can write $\Q^\tg(\d Z) = 
\tilde\Q(\d Z) r(Z_0, Z_\t)$, where 
$r(z_0, z_\t) =\frac{\d \mu_0 \otimes \tg}{\d \tilde\Q_{0,\t}}(z_0, z_\t).$
 From \cite{leonard2014reciprocal}, $\Q^\tg$ is Markov iff $r(x, z_0) = f(x) g(z_0)$ for some $f$ and $g$, which is not the case except the degenerated case ($v_0 = 0$ and $v_0= +\infty$) because $\tilde\Q_{0,1}$ is  not factorized. 
 
 On the other hand, when $v_0 = 0$,  we have that $\Q^x = \Q(\cdot | Z_\t = x)$ is the standard Brownian bridge and hence $\Q^\tg$ is Markov following Proposition~\ref{thm:assumeqx}. 
 When $v_ 0 = +\infty$, as the case of SMLD, 
 $\Q^\tg$ is the law of $Z_t = \rev Z_{\t-t}$ with $
 \d \rev Z_t = \d W_t$ and $\rev Z_0 \sim \tg$, which is also Markov. 
\end{proof}

\begin{proof}[Proof of Proposition~\ref{thm:qtgis}]
Note that 
$$
\Q^\tg(\cdot) = \int \pi(\d x) \Q^x(\cdot)  
= \int  \pi(\d x) \mu(\d z_0~|~x) \tilde \Q^{z_0,x}(\cdot). 
$$
Hence if $\tilde \Q$ is Markov, $\Q^\tg$ is  reciprocal by Definition~\ref{def:reciprocal}. 

On the other hand, if $\Q^\tg$ is reciprocal, we have $\Q^\tg(\cdot) = \int \M^{z_0, x}(\cdot) \mu(\d z_0, \d x) $ for some Markov process $\M$ and probability measure $\mu$ on $\Omega\times \Omega$.  
In this case, we have $\Q^x(\cdot)=\Q^\tg(\cdot | Z_\t = x) =   \int \M^{z_0, x} (\cdot ) \mu(\d z_0~|~x)$, assuming it exits. 
\end{proof}
 
\subsection{Condition for $\Omega$-bridges}

\begin{pro}\label{thm:foranyqomega}
For any $\Q^\Omega$ following $\d Z_t = \eta^\Omega(Z_t,t)\dt + \d W_t$ that is an $\Omega$-bridge, the  $\P^\theta$ in \eqref{equ:omegabridge} 
is also an $\Omega$-bridge if $\E_{Z\sim \barQ^\Omega}[\int_0^\t\norm{f^\theta(Z_t, t)}_2^2\dt ]<+\infty$ and $\KL(\Q^\Omega_0~||~\P^\theta_0) < +\infty$. 
\end{pro}
\begin{proof}[Proof of Proposition~\ref{thm:foranyqomega}] 
We know that $$\KL(\Q^\Omega~||\P^\theta) = 
\KL(\Q^\Omega_0~||~\P^\theta_0) + \frac{1}{2} \E_{Z\sim \barQ^\Omega}\left [\int_0^\t\norm{f^\theta(Z_t, t)}_2^2\dt \right ] < +\infty.$$ This means that $\Q^\Omega$ and $\P^\theta$ are absolutely continuous to each other, and hence have the same support. Therefore, $\Q^\Omega(Z_\t\in \Omega) = 1$ implies that $\P^\theta(Z_\t\in \Omega) = 1$.
\end{proof}

\subsection{Examples of $\Omega$-Bridges}  %
If $\Omega$ is a product space, 
the integration can be factorized into one-dimensional integrals. So it is sufficient to focus on 1D case. 

If $\Omega$ is a discrete set, say $\Omega = \{e_1\ldots, e_K\}$, we have 
\bb 
\eta_{\mathrm{bb}, \varsigma}^\Omega (z, t)
 & = \varsigma_t^2\frac{1}{\sum_{k=1}^K \omega(e_k, z, t)}
 \sum_{k=1}^K %
 \omega(e_k, z, t)
 \frac{e_k - z}{\beta_\t - \beta_t} \\ 
 & = \varsigma_t^2 
 \dd_z  \log  \sum_{k=1}^K \omega(e_k, z, t),
\ee 
where $$
\omega(e_k, z, t) = \exp\left (- \frac{\norm{z - e_k}^2}{2(\beta_\t - \beta_t)}\right ).$$

If $\Omega = [a,b]$, we have 
\bb 
\eta^\Omega_{\mathrm{bb}, \varsigma}(z,t)
& = 
\varsigma_t^2 \frac{1}{\int_a^b \omega(e, z, t)}
\int_{a}^b \omega(e, z, t)   \frac{e - z}{\beta_\t - \beta_t} \d e \\ 
& =\varsigma_t^2 \dd_z \log \int_{a}^b \omega(e, z, t) \d e \\
& = \varsigma_t^2 \dd_z \log \left (F(\frac{z-a}{\sqrt{\beta_\t - \beta_t}}) - 
F(\frac{z-b}{\sqrt{\beta_\t - \beta_t}}) \right ), 
\ee 
where $F$ is the standard Gaussian CDF.

\subsection{Time-Discretization Error Analysis (Proposition~\ref{thm:disc})}

\begin{pro}%
\label{thm:disc_dd}
Assume $\Omega=\RR^d$ and  $\sigma(z, t) = \sigma(t)$ is state-independent and $\sigma(t)>c>0$, $\forall t\in[0,\t]$.  
Take the uniform time grid $\tau^{\mathrm{unif}} \defeq \{i\epsilon\}_{i=0}^K$ with step size $\epsilon=\t/K$ in %
the sampling step \eqref{equ:disc_inference}.
Let $\L_\epsilon(\theta) = \E_{Z\sim \Q^\tg}[\ell_\epsilon(\theta; Z)]$ with  
$$
\ell_\epsilon(\theta, Z_t) 
= - \log p_0^\theta(Z_0) + 
 \frac{1}{2K}\sum_{k=1}^K \norm{\sigma_k^{-1}(s^\theta(Z_{\tk}, \tk) - \eta^{Z_\t}(Z_{[0,\tk]}, \tk))}^2_2,  
$$
where $\epsilon>0$ is a step size with $\t = K\epsilon$ and 
$\tk = (k-1)\epsilon$,   
and $\sigma_k^{2} \defeq (t_{k+1}-\tk)^{-1}\int_{\tk}^{t_{k+1}} \sigma(t)^{2} \dt$.
Let $\P^{\theta, \epsilon}_\T$ be the distribution of the sample $\hat \X_\t$ resulting from the following Euler method:
$$
\hat Z_{t_{k+1}} 
 = \hat Z_{t_k} + \epsilon s^{\theta}(Z_{t_k}, t_k) + \sqrt{\epsilon} \sigma_k \xi_k, 
$$
where $\xi_k \sim \normal(0, I_{d})$ is the standard Gaussian noise in $\RR^d$.  %
Let  $\theta\true$ be an optimal parameter 
satisfying \eqref{equ:global}.  Assume 
$C_0\defeq \sup_{z,t}\left (\norm{s^\thetat(z,t)}^2/(1+\norm{z}^2),~ \trace(\sigma^2(z,t)),~  \E_{\P^{\theta^*}}[\norm{Z_0}^2]\right) <+\infty$, and $s^\thetat$ satisfies 
$\norm{s^\thetat(z,t) - s^\thetat (z',t')}^2_2 \leq L\left ( 
\norm{z -z'}^2 + \abs{t - t'} \right )$ for $\forall z,z'\in \RR^d$ and $t,t'\in[0,\t]$. %
Then we have 
$$
\sqrt{\KL(\tg~||~\P^{\theta,\epsilon}_\t )} \leq 
\sqrt{\L_\epsilon(\theta) - \L_\epsilon(\theta\true)}  + \bigO{\sqrt{\epsilon}}. 
$$
\end{pro}

\begin{proof}[Proof of Proposition~\ref{thm:disc}]
This is the result of Lemma~\ref{lem:lethbe} below by noting that the $\hat \P^\theta$ there is equivalent to the Euler method above, and 
$\L_\epsilon(\theta) - \L_\epsilon(\theta^*) \leq \tilde \L_\epsilon(\theta) - \tilde \L_\epsilon(\theta^*) $ (because $
\sigma_k^{-2} = ((t_{k+1} - t_k)^{-1}\int_{\tk}^{t_{k+1}} \sigma(t)^{2})^{-1} \leq (t_{k+1} - t_k)^{-1}\int_{\tk}^{t_{k+1}} \sigma(t)^{-2}$). 
\end{proof}

\begin{lem}\label{lem:lethbe} 
Let $h$ be a step size and $\eps  = T/K$ for a positive integer $K$. 
For each $t\in [0,\infty)$, denote by $\disc t \eps  = \max(\{ k \eps  \colon k \in \mathbb N \}\cap [0, t]$). 
Assume 
\bb 
&\Qt: ~~~~~\d Z_t = \etat(Z_{[0,t]}, t)\dt + \sigma(Z_t, t) \dW_t,~~~ Z_0 \sim \Q_0\\
& \Ptt: ~~~~~\d Z_t = \stt(Z_t, t)\dt + \sigma(Z_t, t) \dW_t,~~~ Z_0 \sim \Q_0, \\ 
&  \P^\theta: ~~~~~\d Z_t = s^\theta(Z_t, t)\dt + \sigma(Z_t, t) \dW_t,~~~ Z_0 \sim \P^\theta_0 \\
&  \hat \P^{\theta}: ~~~~~\d Z_t = s^\theta(Z_{\disc t \eps }, t)\dt + \sigma(Z_t, t) \dW_t,~~~ Z_0 \sim \P^\theta_0, 
\ee 
where $\Ptt$ is the Markovianization of $\Qt$, and $\hat \P^\theta$ is a discretized version of $\P^\theta$. 
Define 
$$
\tilde \L_\eps(\theta) = 
\E_{\Qtt} \left [ - \log p_0^\theta(Z_0) + 
\frac{1}{2} \int_0^\T 
\norm{ \sigma^{-1}(Z_t, t)(s^{\theta}(\X_{\disc t \eps}, \disc t \eps) -
 \eta^{Z_\t}(\X_{[0,\disc t \eps]}, \disc t \eps ))}^2
 \dt \right]. 
$$
Assume the conditions of Lemma~\ref{lem:bound} holds for $\P^{\theta\true}$, 
and $\sigma(z, t)\geq c>0$ for all $z, t$, and $s^\thetat$ satisfies 
$\norm{s^\thetat(z,t) - s^\thetat (z',t')}^2_2 \leq L\left ( 
\norm{z -z'}^2 + \abs{t - t'} \right )$ for $\forall z,z'\in \RR^d$ and $t,t'\in[0,\t]$. 
Then  
$$
\sqrt{\KL(\Ptt ~||~ \hat \P^\theta)}
\leq \sqrt{\tilde \L_\eps(\theta) -\tilde  \L_\eps(\theta\true)} + 
C \sqrt{\epsilon}, 
$$
where %
$C$ is a constant that is independent of $\epsilon$. 
\end{lem}
\begin{proof}
Define $\norm{f}_{\Q, \sigma}^2 = \E_{Z\sim \Q}[ \int_0^T \norm{\sigma(Z,t) f(Z,t)}^2]$ for convenient notation.
 Let $s^\theta_\epsilon(Z,t) = s^\theta(Z_{\disc t \eps}, \disc t \eps)$, and $\eta_\epsilon = \eta^{Z_T}(Z_{[0,\disc t \eps]}, \disc t \eps)$. 
\bb
& \KL(\Ptt~||~ \hat \P^\theta)  \\
& = \KL(\Ptt_0~||~ \hat \P^\theta_0) + 
\frac{1}{2}  \norm{s^{\theta\true} - s^\theta_\epsilon}^2 \\ 
& \leq 
\KL(\Ptt_0~||~ \P^\theta_0) +
\frac{1}{2} 
\left( (1+\omega)  \norm{s^{\theta}_\epsilon - s^{\theta\true}_\epsilon}^2_{\Ptt, \sigma}
+ (1+1/\omega)\norm{s^{\theta\true} - s^{\theta\true}_\epsilon}^2_{\Ptt, \sigma}  \right)
 \\ 
& \defeq  (1+\omega) I_1 + 
(1+1/\omega)  I_2, 
 \ee 
 where  $\omega >0$ is any positive number and 
 \bb
 I_1 & 
 \defeq  
\frac{1}{1+\omega} \KL(\Ptt_0~||~ \P^\theta_0) + \frac{1}{2}\norm{s^{\theta}_\epsilon - s^{\theta\true}_\epsilon}^2_{\Ptt, \sigma} \\ 
 & =  
\frac{1}{1+\omega} \KL(\Ptt_0~||~ \P^\theta_0) + \frac{1}{2}\norm{s^{\theta}_\epsilon - s^{\theta\true}_\epsilon}^2_{\Qt, \sigma} \\ 
 & \leq   
 \KL(\Ptt_0~||~ \P^\theta_0) + \frac{1}{2}\norm{s^{\theta}_\epsilon - s^{\theta\true}_\epsilon}^2_{\Qt, \sigma} \\ 
  & \leq   
 \KL(\Ptt_0~||~ \P^\theta_0) + 
 \frac{1}{2}
 \left(\norm{s^{\theta}_\epsilon - \eta^{Z_\t}_\epsilon}^2_{\Qt, \sigma} 
 - \norm{s^{\theta\true}_\epsilon - \eta^{Z_\t}_\epsilon}^2_{\Qt, \sigma} 
 \right ) \ant{Lemma~\ref{lem:vb}}\\ 
& = \tilde   \L_\epsilon(\theta) -\tilde  \L_\epsilon(\theta\true), 
 \ee 
 and 
 \bb 
 I_2
 & \defeq \frac{1}{2} \norm{s^{\thetat} - s^\thetat_\epsilon}_{\Ptt, \sigma}^2 \\
  & \leq  \frac{L}{2}   \E_{\Ptt}\left[\int_0^\t \sigma^{-2}(Z_t, t) 
  \left (\norm{Z_t - Z_{\disc t \epsilon}}^2 + (t-\disc t \epsilon)\right )\dt \right] \\  
  & \leq  \frac{L}{2c^2} \E_{\Ptt}\left[\int_0^\t 
  ((Z_t - Z_{\disc t \epsilon})^2 + (t-\disc t \epsilon))\dt \right] \\    
& \leq \frac{L}{2c^2}\left (C_{\Ptt} +1 \right )\int_0^\t (t-\disc t \eps) \dt   \ant{Lemma~\ref{lem:bound}} \\
& = \frac{L}{2c^2}\left (C_{\Ptt} +1 \right )  \frac{\t \eps}{2}. \ant{Lemma~\ref{lem:tth},}
 \ee 
 where $C_{\Ptt}$ is a constant depending on $\Ptt$ that comes from Lemma~\ref{lem:bound}. 
Hence 
\bb 
\KL(\Ptt ~||~\hat \Pt) 
& \leq  \inf_{\omega\geq 0} (1+\omega) I_1 + 
(1+1/\omega)  I_2 \\
& = ( \sqrt{I_1} + \sqrt{I_2} )^2 \\
& \leq \left (\sqrt{\tilde  \L_\eps (\theta) - \tilde  \L_\eps (\theta^*)} +  \frac{1}{2}\sqrt{\frac{L}{c^2}\left (C_{\Ptt} +1 \right ) \t \eps }  \right). 
\ee 
This completes the proof. 
\end{proof}

\begin{lem} 
For any $a, b\in \RR^d$,  and $\omega \geq 0$, 
$$
\norm{a+b}^2_2 \leq (1+\omega) \norm{a}^2_2 +(1+1/\omega) \norm{b}_2^2. 
$$
\end{lem} 
\begin{proof}
\bb 
 (1+\omega) \norm{a}^2_2 +(1+1/\omega) \norm{b}_2^2  
 \geq  
\norm{a}_2^2 + \norm{b}_2^2  + 2 a \tt b  
 = \norm{a+b}^2_2
\ee 
\end{proof}

\begin{lem}\label{lem:tth}
Assume $\t \geq 0$, $\eps \geq 0$ and $\t/\eps  \in \mathbb N$. We have 
$$\int_0^\t   (t - \disc t \eps ) \dt  
 = \frac{\T \eps }{2}.$$
\end{lem}
\begin{proof}
\bb 
\int_0^\t   (t - \disc t \eps ) \dt   
&  = \sum_{k=0}^{K-1} \int_0^\eps  (hk + x -  hk) \dx  \\
& = \sum_{k=0}^{K-1} \int_0^\eps  x\dx \\
& = K \eps^2/2 \\
& = \T \eps /2. 
\ee 
\end{proof}

\begin{lem}[Grönwall's inequality]
Let $I$ denote an interval of the real line of the form $[a, \infty)$ or $[a, b]$ or $[a, b)$ with $a<b$. Let $\alpha,\beta$ and $u$ be real-valued functions defined on $I$. Assume that $\beta$ and $u$ are continuous and that the negative part of $\alpha$ is integrable on every closed and bounded subinterval of $I$.

(a) If $\beta$ is non-negative and if $u$ satisfies the integral inequality
$${\displaystyle u(t)\leq \alpha (t)+\int _{a}^{t}\beta (s)u(s)\,\mathrm {d} s,\qquad \forall t\in I,}$$
which is true if 
$$
u'(t) \leq \alpha'(t) + \beta(t) u(t). 
$$
then
$${\displaystyle u(t)\leq \alpha (t)+\int _{a}^{t}\alpha (s)\beta (s)\exp {\biggl (}\int _{s}^{t}\beta (r)\,\mathrm {d} r{\biggr )}\mathrm {d} s,\qquad t\in I.}$$
(b) If, in addition, the function $\alpha$ is non-decreasing, then
$${\displaystyle u(t)\leq \alpha (t)\exp {\biggl (}\int _{a}^{t}\beta (s)\,\mathrm {d} s{\biggr )},\qquad t\in I.}$$
\end{lem}

\begin{lem}\label{lem:bound}
Consider 
$$
\dX_t = b(\X_t, t)\dt + \sigma(\X_t, t)\dW_t,~~~~~ X_0 = 0, t \in [0, \T]. 
$$
Assume %
there exists a finite constant $C_0$, such that 
$$\norm{b(x, t)}_2^2 \leq C_0 (1+\norm{x}_2^2), ~~~\forall x\in 
\RR^d, ~~ t\in[0,\T], $$
$$
\trace(\sigma\sigma\tt (x,t)) \leq C_0,~~~\forall x\in \RR^d, t\in [0,\T].
$$
and $\E[\norm{\X_0}_2^2] \leq C_0$.

Then for any $0\leq s \leq t \leq \T$, we have 
$$
\E[\norm{\X_t - \X_s}_2^2] \leq K_{C_0, T}  (t-s),
$$
where $K_{C_0, T}$ is a finite constant that depends on $C_0$ and $T$. 
\end{lem}
\begin{proof}
Let $\eta = \sup_{x,t} \trace(\sigma \sigma\tt(x, t)).$ We have by Ito Lemma, 
\bb
\frac{\d}{\dt }\E\left [\norm{\X_{t}- \X_s}_2^2 \right ] & 
= \E\left [
2(\X_t - \X_s)\tt (b(\X_t, t) + \d W_t) +  \eta
 \right ]   \\
& 
= \E\left [
2(\X_t - \X_s)\tt b(\X_t, t) +\eta
\right ]   \\
& 
= \E\left [\norm{\X_t - \X_s}_2^2 + \norm{b(\X_t, t)}_2^2 +  d \right ]   \\
& \leq \E\left [\norm{\X_t - \X_s}_2^2 + C_0(1+\norm{\X_t}_2^2) +  \eta \right ]   \\
& \leq (1+2C_0)\E\left [\norm{\X_t - \X_s}_2^2 \right ] +
\eta+ C_0(1 + 2\E\left [\norm{\X_s}_2^2\right ] ). 
\ee 
Using Gronwall's inequality, 
$$
\E\left [\norm{\X_t - \X_s}_2^2\right]
\leq (t-s)(\eta + C_0(1+2 \E[\norm{\X_s}_2^2])) \exp\left (
(t-s)(1+2C_0) \right).  
$$
Taking $s=0$ yields that 
$$
\E\left [\norm{\X_t - \X_0}_2^2\right]
\leq t(\eta + C_0(1+2 \E[\norm{\X_0}_2^2])) \exp\left (
t (1+2C_0) \right).  
$$
Hence 
\bb 
\E[\norm{\X_t}^2_2] 
& \leq 
2\E\left [\norm{\X_t - \X_0}_2^2\right] + 2\E\left [\norm{  X_0}_2^2\right] \\ 
& \leq  2t(\eta + C_0(1+2 \E[\norm{\X_0}_2^2])) \exp\left (
t (1+2C_0) \right) + 2\E\left [\norm{\X_0}_2^2\right]  \\
&\leq  4\T(C_0 + C_0^2) \exp(T(1+2C_0)) + 2 C_0. 
\ee 
Therefore, 
\bb 
& \E\left [\norm{\X_t - \X_s}_2^2\right]\\
& \leq (t-s)(\eta + C_0(1+2 \E[\norm{\X_s}_2^2])) \exp\left ((t-s)(1+2C_0) \right) \\
& \leq C(t-s),
 \ee 
 where 
 \bb 
 C & = (2C_0 
 + 4 C_0^2 + 8C_0\T(C_0 + C_0^2) \exp(T(1+2C_0))) \exp\left (T(1+2C_0) \right). 
 \ee 
\end{proof}

\subsection{Statistical Error Analysis (Proposition~\ref{thm:asymptotic})}

\begin{pro}\label{thm:asymptotic_dd} 
Assume the  conditions in Proposition~\eqref{thm:disc_dd}. 
Assume $\hat \theta_n = \argmin_{\theta} \hat \L_\epsilon(\theta)$ with $\hat \L_\epsilon(\theta)=\sum_{i=1}^n \ell_\epsilon(\theta; Z\datai)/n$, $Z\datai\sim \Q^\tg$. 
Take $\Q^x$ to be the standard Brownian bridge $\d Z^x_t = \frac{x-Z^x_t}{\T-t}\dt + \d W_t$ with $Z_0\sim \normal(0,v_0)$ and $v_0 >0$. 
Assume 
$
\sqrt{n}(\hat \theta_n-\theta\true) \dto 
\normal(0, \Sigma_*)$ as $n\to+\infty$, where $\Sigma_*$ is the asymptotic covariance matrix of the M estimator $\hat \theta_n$. 
Assume $\L_\epsilon(\theta)$ is second order continuously differentiable and strongly convex at  $\thetat$. 
Assume $\tg$ has a finite covariance matrix and  admits a density function $\pi$ 
that satisfies %
$\sup_{t\in[0,\t]}\E_{\Q^\tg}[\norm{\dd_\theta %
s^{\thetat}(Z_t, t)}^2(1+\norm{\dd\log\pi(Z_\t)}^2+\trace(\dd^2\log \pi(Z_\t)))]<+\infty$.
We have 
$$
\E[\sqrt{\KL(\tg ~||~ \P_\t^{\hat \theta_n, \epsilon})}] = \bigO{ \sqrt{\frac{\log (1/\epsilon)+1}{n}} + \sqrt{\epsilon} },
$$
where the expectation is w.r.t. the randomness of $\hat\theta_n$. 
\end{pro}
\begin{proof}[Proof of Proposition~\ref{thm:asymptotic}] 
Let 
\bb 
\theta^* = \argmin_{\theta}
\L_\epsilon(\theta):=\E_{Z\sim \Q^\tg} [\ell(\theta; Z)], &&
\hat\theta_n = \argmin_{\theta}  
\hat \L_\epsilon(\theta):=\frac{1}{n}\sum_{i=1}^n  \ell(\theta; Z\datai), 
\ee 
where $\{Z\datai\}_{i=1}^n$ is drawn i.i.d. from $\Q^\tg$.  
We assume that $\hat \theta_n$ is an asymptotically normal M-estimator, in which case we have 
$$
\sqrt{n} (\theta_n - \theta\true)  \dto  
\normal(0, \Sigma_*), 
$$
where 
\bb 
\Sigma_* = H_*^{-1} V_* H_*^{-1},
&& H_* = \E_{Z\sim \Q^\tg} \left [\dd_{\theta\theta}^2 \ell(\theta\true; Z) \right], 
&&  V_* = \E[ \dd_\theta \ell(\theta\true; Z)  \dd_\theta \ell(\theta\true; Z)\tt ], 
\ee 
and %
\bb 
n \E[(\L(\hat \theta_n) - \L(\theta\true))] 
 \asymp  \left [ \frac{1}{2} \sqrt{n}(\theta_* - \hat \theta_n)\tt H_*  \sqrt{n}(\theta^*-\hat\theta_n) \right]  
 \asymp \frac{1}{2} \trace(H_*^{-1} V_*),
\ee 
where $f\asymp g $ denotes that $f-g = \smallo{1}$.  We now need to bound $ \trace(H_*^{-1} V_*).$
Combining the results in Lemma~\ref{lem:assumelambdamin}
and Lemma~\ref{lem:letpistarbe}, we have when $t_k = (k-1) \epsilon$ and $T = K\epsilon$, 
$$
\trace(H_*^{-1} V_*)
= \bigO{
1 + \frac{1}{K} \sum_{k=1}^K \frac{1}{\t -t_k} 
} 
= \bigO{1 + \log(1/\epsilon)} . 
$$
Hence, 
\bb
\E[\sqrt{\KL(\tg ~||~ \P_\t^{\hat \theta_n, \epsilon})}] & = \bigO{ \E[\sqrt{\L(\hat\theta_n) - \L(\theta\true)}] + \sqrt{\epsilon} } \\ 
& = \bigO{ \sqrt{\E[\L(\hat\theta_n) - \L(\theta\true)]} + \sqrt{\epsilon} }\\
& = \bigO{\sqrt{\frac{\log(1/\epsilon+1)}{n}} + \sqrt{\epsilon}}. 
\ee
\end{proof} 

\begin{lem} 
Assume the conditions in Proposition~\ref{thm:asymptotic_dd}.  
\label{lem:assumelambdamin}
Define 
\bb 
I_0 = \E_{Z\sim \Q^\tg}\left [\norm{\dd \log p_0^{\theta\true}(Z_0)}^{2} \right ], && 
I_k = \E_{Z\sim \Q^\tg} \left [\norm{\dd_\theta s^{\thetat}(Z_{t_k}, t_k)}^2 \trace(\cov(\eta^{Z_\t}(Z_{[0,t_k]}, t_k)~|~Z_{t_k})) \right ], 
\ee 
for $\forall k = 1,\ldots K$. 
Then 
$$
\trace(H_*^{-1}V_*) ^{1/2}
\leq 
\frac{1}{\lambda_{\min}(H_*)^{1/2}} 
\left (I_0^{1/2} + \left(\frac{1}{K}\sum_{k=1}^K I_k\right)^{1/2} \right ). 
$$ 
\end{lem} 
\begin{proof}
From Lemma~\ref{thm:letaandb}, 
$\trace(H_*^{-1} V_*) \leq 
(\lambda_{\min}(H_*))^{-1} \trace(V_*)$. 
Hence we just need to bound $\trace(V_*)$. 

\bb
&\trace(V_*)^{1/2}
 = \E_{Z\sim \Q^\tg}\left [\norm{\dd_\theta \ell(\thetat, Z)}_2^2 \right]^{1/2}  \\
  & \leq 
  \E_{Z\sim \Q^\tg}\left [\norm{\dd_\theta \ell(\thetat, Z)}_2^2 \right]^{1/2}
  + \frac{1}{K}\sum_{k=1}^K 
  \E_{Z\sim \Qt}\left[ \norm{\dd_\theta s^{\theta\true}(Z_{t_k}, t_k)( s^{\theta^*}(Z_{t_k}, t_k) - \eta^{Z_\t}(Z_{[0,t_k]}, t_k) ) }_2^2 \right] ^{1/2}
  \\
  & \leq 
  \E_{Z\sim \Q^\tg}\left [\norm{\dd_\theta \ell(\thetat, Z)}_2^2 \right]^{1/2}
  + \frac{1}{K}\sum_{k=1}^K 
  \E_{Z\sim \Qt}\left[ \norm{\dd_\theta s^{\theta\true}(Z_{t_k}, t_k)}_2^2\norm{(s^{\theta^*}(Z_{t_k}, t_k) - \eta^{Z_\t}(Z_{[0,t_k]}, t_k) ) }_2^2 \right] ^{1/2}\\
  & = 
  \E_{Z\sim \Q^\tg}\left [\norm{\dd_\theta \ell(\thetat, Z)}_2^2 \right]^{1/2}
  + \frac{1}{K}\sum_{k=1}^K 
  \E_{Z\sim \Qt}\left[ \norm{\dd_\theta s^{\theta\true}(Z_{t_k}, t_k)}_2^2
  \trace\left (\cov\left (\eta^{Z_\t}(Z_{[0,t_k]}, t_k)~|~Z_{t_k} \right )\right) 
  \right] ^{1/2}   \\
  & = I_0^{1/2} + \frac{1}{K}\sum_{k=1}^K I_k^{1/2} \\
  & \leq  I_0^{1/2} + \sqrt{\frac{1}{K}\sum_{k=1}^K I_k}. 
\ee 
\end{proof}

\begin{lem}
Assume the results in Lemma~\ref{lem:assumelambdamin} and Lemma~\ref{lem:letpistarbe} hold.   
Assume $$\max_{k\in1,\ldots,K}\E_{Z\sim \tg}\left [\norm{\dd_\theta \ts^{\theta\true}(Z_{t_k}, t_k)}_2^2 \left ( 1 +  \norm{\dd\log \tgd(Z_\t)}_2^2 + \trace(\dd^2 \log \tgd(Z_\t)) \right ) \right ] < +\infty,$$ 
Then for $k=1,\ldots, K$, we have
$I_k 
 = \bigO{\frac{1}{T-t_k} + 1 }.  $
\end{lem}
\begin{proof}
It is a direction application of \eqref{equ:tracecovyt}. 
\end{proof}

\begin{lem}\label{thm:letaandb}
Let $A$ and $B$ be two $d\times d$ positive semi-definite matrices. Then  
$\trace(AB) \leq \lambda_{\max}(A)\trace(B).$
\end{lem}
\begin{proof}
Write $A$ into $A = \sum_{i=1}^d\lambda_i u_i u_i\tt$ where  $\lambda_i$ and $u_i$ is the $i$-th eigenvalue and eigenvectors of $A$, respectively.  Then 
$$
\trace(AB) = \trace(\sum_{i=1}^d \lambda_i u_i \tt B u_i )
\leq  \lambda_{\max}(A)\trace(\sum_{i=1}^d u_i \tt B u_i ) 
= \lambda_{\max}(A)\trace(B).   
$$
\end{proof}

\paragraph{Controlling the Conditional Variance of the Regression Problem}
Assume $\Q^x$ is the standard Brownian bridge: 
\bbb\label{equ:sbb}
\Q^x \colon ~~~~ 
\d Z_t^x = \frac{x - Z^x_t}{\t-t} \dt +  \d W_t,~~~~ Z_0\sim\normal(0, v_0). 
\eee
In this case, the (ideal) loss function is %
\bb  
\L(\theta) = - \E_{X\sim \tg, Z\sim \Q^X} \left [ \log p_0^\theta(Z_0) + \frac{1}{2}\int_0^\t \norm{s^\theta(Z_t, t) - Y_t}^2\dt \right ], 
&&
\text{where $Y_t = \frac{X - Z_t}{1-t}.$} 
\ee 
The second part of the loss is a least square regression for predicting $Y_t = \eta^{X}(Z_{t}, t)$ with $s^\theta(Z_t,t)$.  
The conditioned variance $\cov(Y_t~|~Z_t)$ 
is an important factor that influences the error of the regression problem. 
We now show that  $\trace(\cov(Y_t~|~Z_t)) = O(1/\T-t)$ which means that it explodes to infinity when $t \uparrow \t$. 
 
 First, note that $\trace(\cov(Y_t~|~Z_t)) = \frac{1}{(\T-t)^2}\trace(\cov(X~|~Z_t))$. Using the estimate  in Lemma~\ref{lem:letpistarbe}, we have 
 \bbb \label{equ:tracecovyt} 
\trace(\cov(Y_t~|~Z_t)) = \bigO{\frac{1}{T-t} + 
\E\left [\norm{\dd_x\log\tgd\left (X  \right ) }_2^2 + \trace\left (\dd^2 \log \tgd\left (X \right)\right ) \bigg | Z_t \right ] }. 
\eee  

\begin{lem}
For the standard Brownian bridge in \eqref{equ:sbb}, we have 
$$
Z_t^x \sim \normal\left (\frac{t}{T}x, ~~\frac{t(T-t)}{T} + \frac{(\t-t)^2}{T^2} v_0 \right ). 
$$
\end{lem}
\begin{proof}
Let $Z^{z_0, x}_t$ be the same process that is initialized from $Z_0^{z_0, x} = z_0$. We have from the textbook result regarding Brownian bridge that we can write $Z^{z_0, x}_t 
= \frac{t x + (\t - t) z_0}{\t} + \sqrt{\frac{t(\t-t)}{\t} } \xi_t $ where $\xi_t$ is some standard Gaussian random variable. 
The result  follows 
directly as $Z_t^x = Z^{Z_0, x}_t$ with $Z_0 \sim \normal(0, v_0)$.  
\end{proof}

\begin{lem}\label{lem:letpistarbe}
Let $\pi^*$ be the density function $\tg$ on $\RR^d$ whose covariance matrix exists.  
When $X\sim \tg$ and $Z \sim \Q^X$ from \eqref{equ:sbb} with $v_0>0$.  
Then the density function $\rho_t(x|z)$ of $X |Z_t = z_t$ satisfies
\bbb \label{equ:rho} 
\rho_t(x|z) 
\propto \tgd(x) \exp\left (-\frac{\norm{\frac{\t}{t} z -  x}^2_2}{2(\frac{\t (\T-t)}{t} + v_0 \frac{(\t-t)^2}{t^2}) } \right ).  
\eee  
In addition, there exists positive constants $c<+\infty $ and $\tau\in(0, \T)$, such that 
\bb
\trace(\cov_{\rho_t}(x|z)) 
\leq 
\begin{cases} 
w_t d + w_t^2 
\E_{\rho_t}\left [\norm{\dd_x\log\tgd\left (x  \right ) }_2^2 + \trace\left (\dd^2 \log \tgd\left (x \right)\right ) ~\bigg |~  z \right ], &\text{ when $\tau \leq t\leq  \t$} \\
c, & \text{ when $0\leq t\geq \tau$}, 
\end{cases}
\ee 
where $w_t = {\frac{\t (\T-t)}{t} + v_0 \frac{(\t-t)^2}{t^2}}$. 
So $\trace(\cov_{\rho_t}(x|z))$ is bounded and decay to zero with rate $O(T-t)$ as $t\uparrow T$. 
\end{lem}
\begin{proof}
We know that $X \sim \tg$ and $Z_t^X | X \sim \normal(t/T X, ~w_t)$. Hence, \eqref{equ:rho} is a direct result of Bayes rule. %
Then  Lemma~\ref{lem:steinvar} gives 
\bb
\trace(\cov_{\rho_t}(x|z)   )  
& =   w_t d + w_t^2 
\E_{\rho_t}\left [\norm{\dd_x\log\tgd\left (x  \right ) }_2^2 + \trace\left (\dd^2 \log \tgd\left (x \right)\right ) \bigg | z \right ]. 
\ee 

On the other hand, 
\bb
\rho_t(x|z) \propto \tgd(x) 
\exp( -\frac{1}{2 w_t} \norm{x}^2 + \frac{\t}{t w_t}  z \tt x ),
\ee 
When $t\to 0$, we have $1/w_t \to 0$ and $\t /(t w_t) \to 0$. Hence, $\rho_t(x|z)$ converges to $\tgd(x)$ as $t\to 0$, as a result, $\trace(\cov_{\rho_t}(x|z)) \to \trace(\cov_{\tgd}(x)) < +\infty$.  
Therefore, for any $c>0$, there exists  $t_0 >0$, such that $\trace(\cov_{\rho_t}(x|z)) \leq \trace(\cov_{\tgd}(x)) + c$ when $0\leq t \leq t_0$. 
\end{proof}

\begin{rem}
We need to have $v_0 >0$ to ensure that $\t/(t w_t) \to 0$ in the proof of Lemma~\ref{lem:letpistarbe}. This is purely a technical reason, for yielding a finite bound of the conditioned variance when $t$ is close to $0$. 
We can 
establish the same result when $v_0 = 0$ 
by adding the assumption that $\max_{k\in1,\ldots,K}\E_{Z\sim \Qt}\left [\norm{\dd_\theta s^{\theta\true}(Z_{t_k}, t_k)}_2^2 \trace(\cov_{\Pi^*_{Z_{t_k}}}(Z_\t)) \right ] < +\infty$, where $\Pi_z^*$ is the distribution with density $\pi^*_z(x)  \propto \tgd(x) \exp(z\tt x/\t)$. 
\end{rem}

\begin{lem}\label{lem:steinvar} 
Let $p(x) \propto \pi(x) \exp\left (  -\alpha \frac{\norm{x -b}_2^2}{2}\right)$ be a positive probability density function on $\RR^d$, 
where $\alpha > 0$, $b\in \RR$ and $\log \pi$ is continuously second order differentiable. Then 
\bb 
\trace(\cov_p(x))\leq 
\E_p[\norm{x}^2_2] 
=\alpha^{-1}   d 
+ \alpha^{-2} 
\left (
\E_p[\norm{\dd_x\log\pi(x) }_2^2 + \trace(\dd^2 \log \pi(x)) ] 
\right). 
\ee 
\end{lem}
\begin{proof} 
Let us focus on the case when $b = 0$ first. 
Stein's identity says that 
\bb
\E_p\left  [
(\dd_x \log \pi(x) - \alpha x)\tt  \phi(x) + \dd_x \tt \phi(x) 
\right ] = 0, 
\ee 
for a general  continuously differentiable function $\phi$ when the integrals above are finite.  %

Taking $\phi = x$ yields that 
\bb
\E_p\left  [(\dd_x \log \pi(x) - \alpha x)\tt  x + d  \right ] = 0, 
\ee 
which gives 
$$
\E_p[\norm{x}_2^2] = \alpha^{-1}(\E_p[\dd_x\log \pi(x)\tt x] + d). 
$$
\myempty{Let $g = \E_p[\norm{\dd_x \log \pi(x)}_2^2]^{1/2}$. 
This gives 
$$
(\E_p[\norm{x}_2^2]^{1/2} - \alpha^{-1} g/2 )^2 
\leq \alpha^{-2}  g^2/4 + \alpha^{-1} d.   
$$
This gives 
$$
\E_p[\norm{x}_2^2]^{1/2} \leq 
\alpha^{-1} g/2
+ \sqrt{\alpha^{-2}  g^2/4 + \alpha^{-1} d} . 
$$
Or
\bb 
\E_p[\norm{x}_2^2]  
&\leq \alpha^{-2} g/4+ \alpha^{-2}  g^2/4 + \alpha^{-1} d  + 2\alpha^{-1} g/2\sqrt{ \alpha^{-2}  g^2/4 + \alpha^{-1} d } \\
& = \alpha^{-1} d  + \alpha^{-2} g/2  + \alpha^{-1} g\sqrt{ \alpha^{-2} g^2/4 + \alpha^{-1} d }. 
\ee }
On the other hand, 
taking $\phi(x) = \dd_x \log \pi(x)$ yields 
\bb
\E_p\left  [(\dd_x \log \pi(x) - \alpha x)\tt  \dd_x \log \pi(x) + \trace(\dd^2 \log \pi(x)) \right ] = 0, 
\ee 
which gives 
$$
\E_p[\dd_x\log\pi(x)\tt x] = \alpha^{-1} \left (
\E_p[\norm{\dd_x\log\pi(x) }_2^2 + \trace(\dd^2 \log \pi(x)) ] 
\right). 
$$
This gives 
\bb 
\E_p[\norm{x}^2_2] 
= d \alpha^{-1}  
+ \alpha^{-2} \left (
\E_p[\norm{\dd_x\log\pi(x) }_2^2 + \trace(\dd^2 \log \pi(x)) ] 
\right). 
\ee 

For $b\neq 0$,
define $\tilde p(x ) \propto \pi\left (x + b\right) \exp\left ( - \frac{\alpha}{2} \norm{x}^2 \right)$, which is the distribution of $\tilde x = x-b$ when $x\sim p$. 
Then applying the result above to $\tilde p$ yields 
\bb
\trace(\cov_{p}(x)   )  
& = \trace(\cov_{\tilde p}(x) ) \\ 
& \leq   \alpha^{-1} d + \alpha^{-2}  
\E_{x\sim \tilde p}\left [\norm{\dd_x\log\pi\left (x + b \right ) }_2^2 + \trace\left (\dd^2 \log \pi\left (x + b \right)\right )   \right ] \\ 
& =   \alpha^{-1} d + \alpha^{-2}  
\E_{\sim p}\left [\norm{\dd_x\log\pi\left (x  \right ) }_2^2 + \trace\left (\dd^2 \log \pi\left (x \right)\right )  \right ]. 
\ee 
\end{proof}

\section{Additional Materials of the Experiments}
In our experiments, $T=1$ and $\epsilon = T/K = 1 / K$. Moreover, 
we take the time grid by randomly sampling from $\{i/K\}_{i=0}^{K-1}$ for the training objective Eq.~\eqref{equ:disc_training}. For evaluation, 
we calculate the standard  evidence lower bound (ELBO) by viewing the resulting time-discretized model as a latent variable model: 
\begin{equation*}
    \E_{X\sim \tg} [- \log \hat p^\theta_T (
    X)] \leq  
    \E_{Z\sim \Qt}\left[ - \log \frac{\hat p^\theta_0(Z_0)}{q_0(Z_0)} - \sum_{k= 1}^{K} \log \frac{\hat p^\theta_{t_{k+1}|t_k} ( Z_{t_{k+1}} |  Z_{t_k})}{q_{t_{k+1}|t_k}( Z_{t_{k+1}} |  Z_{t_k})} \right],
\end{equation*}
where $t_k = (k-1) \epsilon$, and 
$\hat p^\theta$ is the density function of the time-discretized version of $\P^\theta$, and $q$ is the density function of $\Q$. 
We adopt Monte-Carlo sampling to estimate the log-likelihood. As in~\citep{song2020score}, we repeat 5 times in the test set for the estimation. For categorical/integer/grid generation, the likelihood of the last step should take the rounding into account: in practice, we have $\hat Z_{T} = \mathrm{rounding}(\hat Z_{t_{K}} + \epsilon s^{\theta}(\hat Z_{t_K}, t_K) + \sqrt{\epsilon} \sigma(Z_{t_k}, t_K) \xi_K, ~ \Omega)$, where $\mathrm{rounding}(x, \Omega)$ denotes finding the nearest element of $x$ on $\Omega$, and hence the likelihood $\hat p^\theta_{T|t_{K}}$  of the last step should incorporate  the rounding operator as a part of the model.  

\subsection{Generating Integer-valued Point Clouds} 
In this experiment, we need to process point cloud data on integer grid. 
To prepare the data, we firstly sample 2048 points from the ground truth mesh. 
Then, we normalize all the point clouds to a unit bounding box. 
After this, we simply project the points onto grid point by rounding the coordinate to integer. 
The metrics in the main text, MMD, COV and 1-NNA are computed with respect to the post-processed integer-valued training point clouds.

\subsection{Generating Semantic Segmentation Maps on CityScapes}
In this experiment, we set \emph{(Noise Decay A)}: $\varsigma^2_t=3\exp(-3t)$;  \emph{(Noise Decay B)}: $\varsigma^2_t=3(1-t)$; \emph{(Noise Decay C)} $\varsigma^2_t = 3 - 3\exp(-3(1-t))$. We visualize the noise schedule in Figure~\ref{fig:appendix_varsigma}. 
Note that, except for Constant Noise, all the other three processes gradually decrease the magnitude of the noise as $t \rightarrow 1$.
For fair comparison, we use the same neural network as in~\cite{hoogeboom2021argmax}. The network is optimized with Adam optimizer with a learning rate of $0.0002$. 
The model is trained for 500 epochs.
The CityScapes  dataset~\citep{Cordts2016Cityscapes} contains photos captured by the cameras on the driving cars.
A pixel-wise semantic segmentation map is labeled for each photo.
As in~\citep{hoogeboom2021argmax}, we rescale the segmentation maps from cityscapes to $32 \times 64$ images using nearest neighbour interpolation.
Our training set and test set is exactly the same as that of \citep{hoogeboom2021argmax} for fair comparison.
We provide more samples in Figure~\ref{fig:appendix_segmentation}.

\begin{figure}
    \centering
    \includegraphics[width=0.4\textwidth]{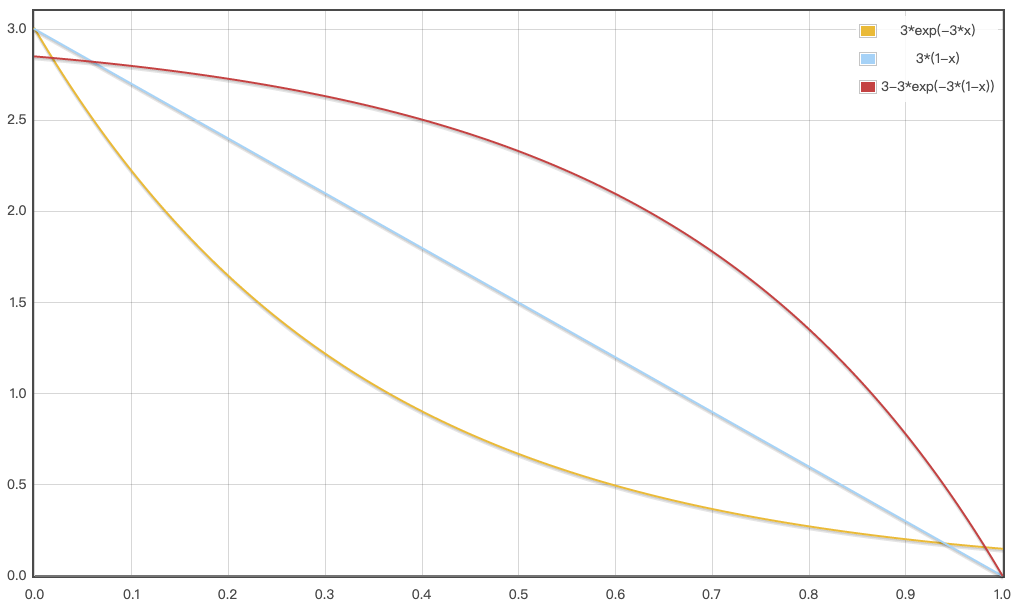}
    \caption{Visualization of the noise schedule of Noise decay A, Noise decay B and Noise decay C.}
    \label{fig:appendix_varsigma}
\end{figure}

\begin{figure}
    \centering
    \includegraphics[width=0.5\textwidth]{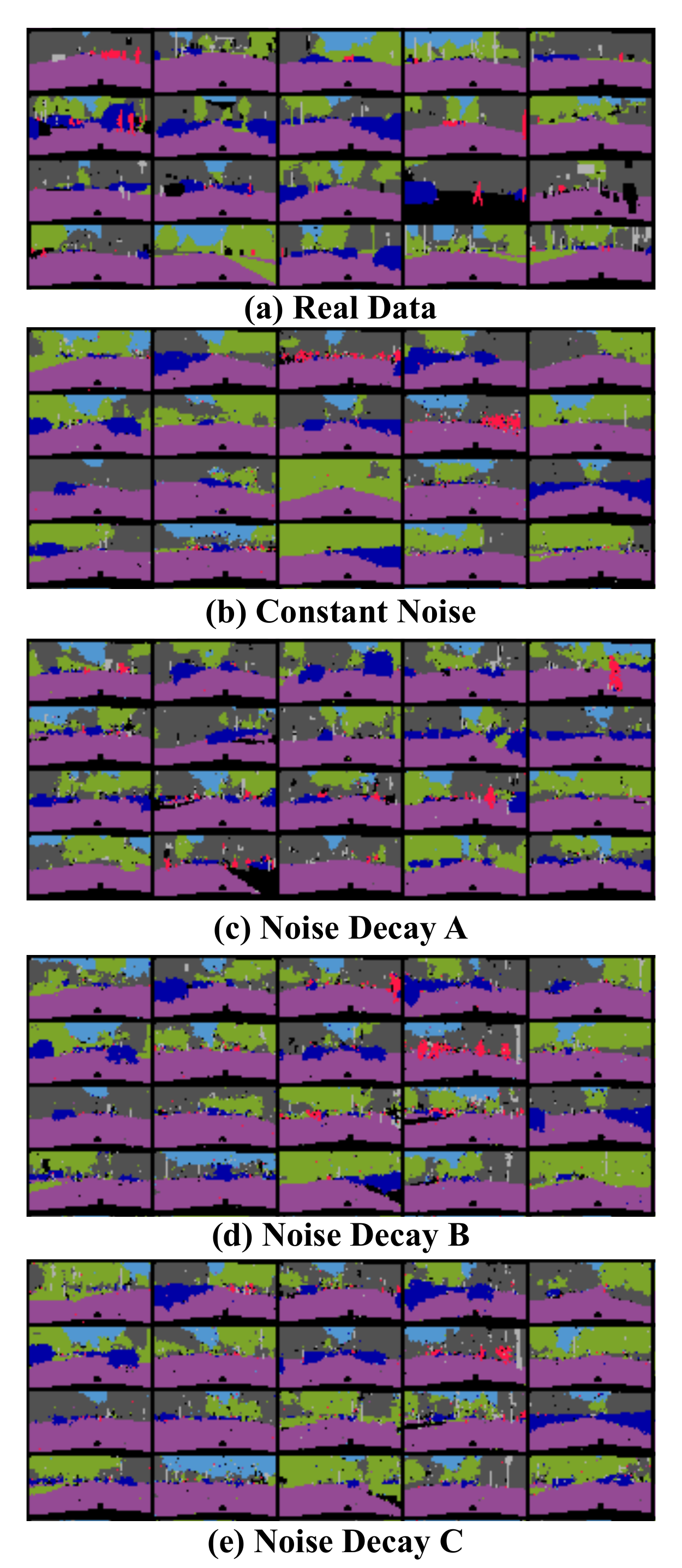}
    \caption{Additional samples from real data, Constant Noise, Noise decay A, Noise decay B and Noise decay C.}
    \label{fig:appendix_segmentation}
\end{figure}

\begin{figure}
    \centering
    \includegraphics[width=0.4\textwidth]{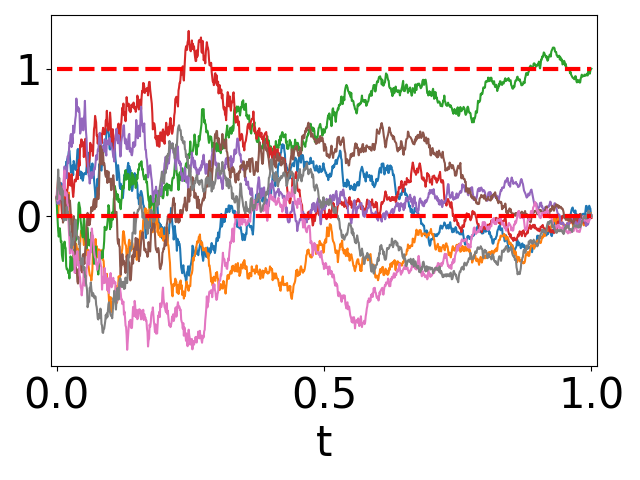}
    \caption{Diffusion process of one pixel (a $8$-dimensional vector) in CityScapes. As $t \rightarrow 1$, 7 of the dimensions reaches $0$, while 1 of the dimensions reaches $1$, turning the vector into a \texttt{one-hot} vector.}
    \label{fig:appendix_cityscape_trajectory}
\end{figure}

\subsection{Continuous CIFAR10 generation}
We provide additional results on generating CIFAR10 images in the continuous domain.
The model is trained using the same training strategy  as DDPM~\citep{ho2020denoising} with the code base provided in~\citep{song2020score}. Specifically, the neural network is the same U-Net structure as the implementation in~\citep{song2020score}. The optimizer is Adam  with a learning rate of $0.0002$. According to common practice~\citep{song2020improved, song2020score}, the training is smoothed by exponential moving average (EMA) with a factor of $0.999$. The results are shown in Figure~\ref{fig:appendix_cifar_continuous} and Table~\ref{tab:continuous_cifar10}. We use $K=1000$ and $\d t=0.001$ for discretizing the SDE. Bridge variants yields similar generation quality as other diffusion models.

\begin{figure}
    \centering
    \includegraphics[width=0.99\textwidth]{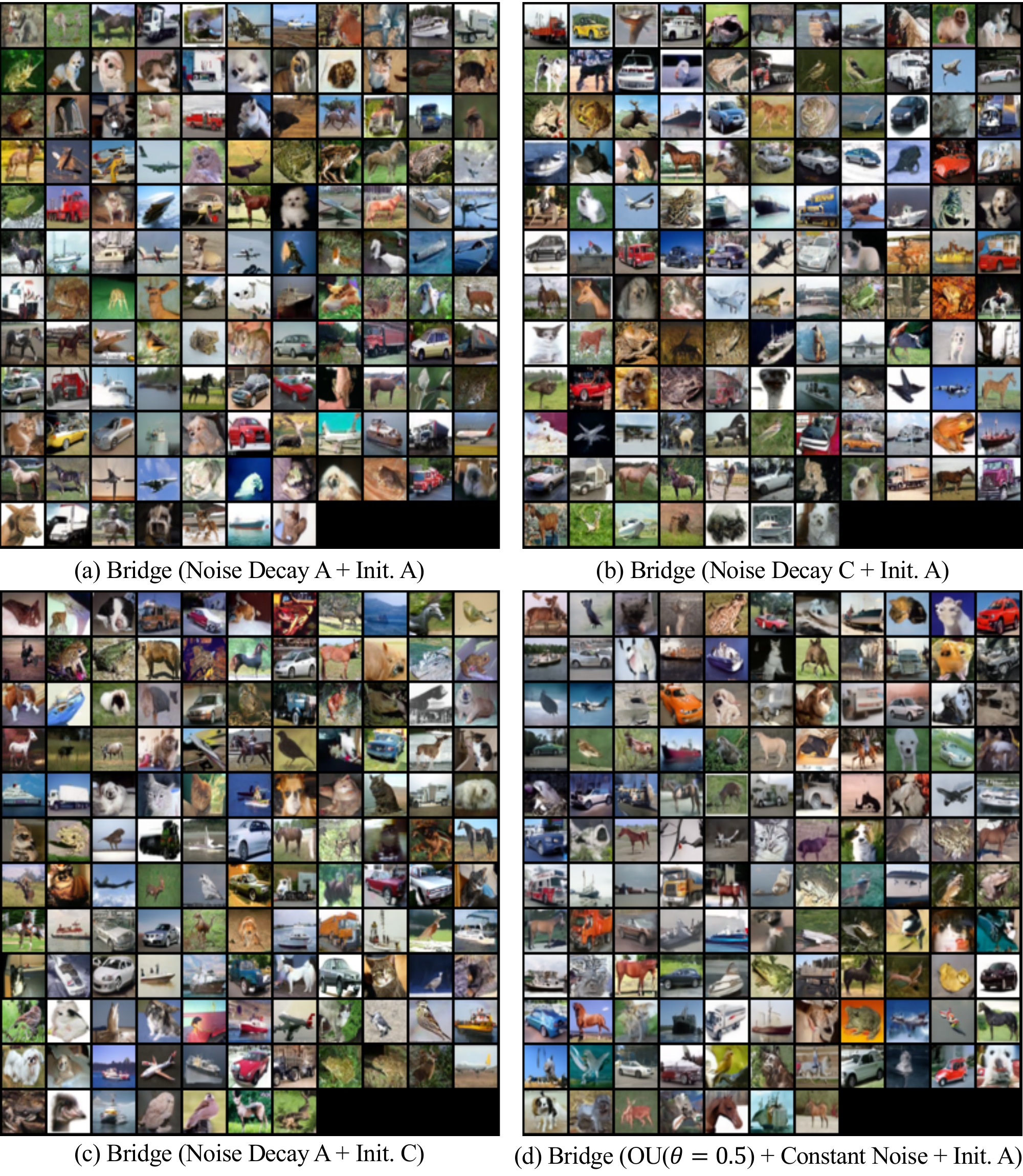}
    \caption{Samples for continuous CIFAR10 generation. Bridge variants can generate high-quality images.}
    \label{fig:appendix_cifar_continuous}
\end{figure}

\begin{figure}
    \centering
    \includegraphics[width=0.99\textwidth]{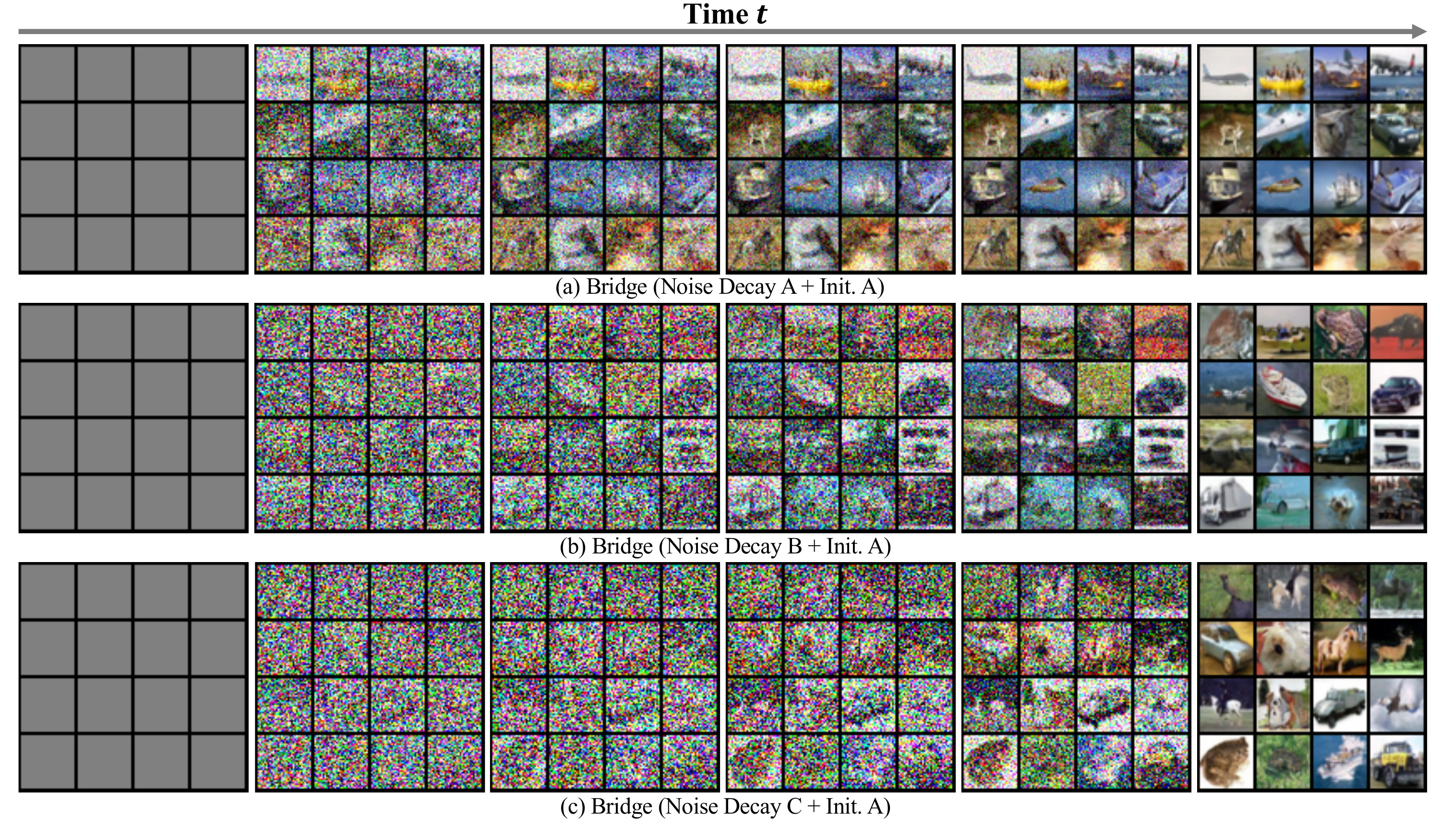}
    \caption{Sampling trajectory for continuous CIFAR10 generation with Noise Decay A, Noise Decay B, Noise Decay C. Bridge can generate images with different noise schedules.}
    \label{fig:appendix_noise}
\end{figure}

\subsection{Discrete CIFAR10 generation}
The experiment details are similar to the continuous CIFAR10 generation, except that the domain of generation is limited to integer values. To account for the discretization error, after the final step, we apply rounding to the generated images to get real integer-valued images. We compare the value distribution of the generated images in Figure~\ref{fig:appendix_value_cifar}.

\begin{figure}
    \centering
    \includegraphics[width=0.5\textwidth]{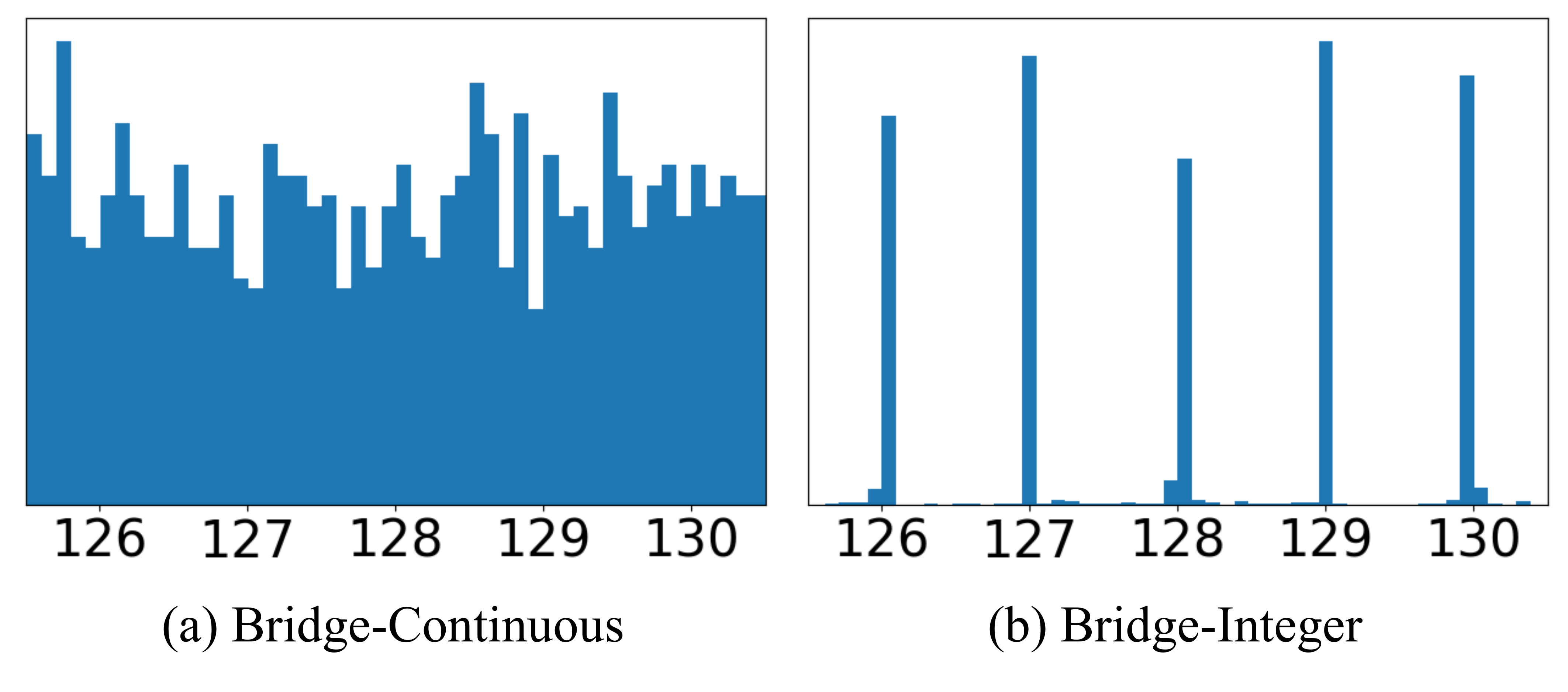}
    \caption{Final value distribution of the generated images with Bridge-Continuous and Bridge-Integer (before rounding) on CIFAR10. We only show the values in $[125.5, 130.5]$ for visual clarity. Our Bridge-Integer generates discrete values.}
    \label{fig:appendix_value_cifar}
\end{figure}

\begin{table}[]
    \centering
    \begin{tabular}{l|ccc}
    \hline \hline
    Methods & IS & FID   \\ \hline 
    \multicolumn{1}{l}{\textbf{Conditional}} &  &  \\ \hline \hline
    EBM~\citep{du2019implicit} & 8.30 & 37.9  \\
    JEM~\citep{grathwohl2019your} & 8.76 & 38.4  \\
    BigGAN~\citep{brock2018large} & 9.22 & 14.73  \\
    StyleGAN2+ADA~\citep{karras2020training} & \textbf{10.06} & \textbf{2.67} \\ \hline \hline
    \multicolumn{1}{l}{\textbf{Unconditional}} & \multicolumn{1}{c}{} &\\ \hline \hline
    NCSN~\citep{song2019generative} & 8.87 & 25.32 \\
    NCSNv2~\citep{song2020improved} & 8.40 & 10.87 \\
    DDPM ($L$)~\citep{ho2020denoising} & 7.67 & 13.51 \\
    DDPM ($L_{simple}$)~\citep{ho2020denoising} & 9.46 & 3.17 \\ 
    Schr{\"o}dinger~\citep{wang2021deep} & 8.14 & 12.32 \\ \hline \hline
    Bridge (Constant Noise + Init. A) & 8.22 & 9.80 \\ 
    Bridge (Noise Decay A + Init. A) & 8.83 & 6.76 \\
    Bridge (Noise Decay B + Init. A) & \color{blue}{\textbf{8.84}} & 6.52  \\ 
    Bridge (Noise Decay C + Init. A) & 8.62 & 7.62  \\ \hline
    Bridge (Noise Decay A + Init. B ) & 8.82 & 7.21  \\
    Bridge (Noise Decay A + Init. C ) & 8.75 & \color{blue}{\textbf{6.28}} \\
    Bridge (OU($\theta=0.1$) + Constant Noise + Init. A) & 8.15 & 10.94 \\
    Bridge (OU($\theta=0.5$) + Constant Noise + Init. A) & 8.19 & 10.85 \\ \hline 
    \hline \hline
    \end{tabular}
    \caption{Additional results on continuous CIFAR10 generation. `Init.' refers to `Initialization'. Bridge yields comparable IS and FID with other diffusion models.}
    \label{tab:continuous_cifar10}
\end{table}

\section{Broader Impact}
This paper focuses on theoretical analysis of diffusion models.
In terms of social impact, 
our method aims to open the black-box of diffusion generative models, and hence increase the interpretability and reliability of this family of ML models. 
Yet, it is not impossible to use the generative models for the generation of harmful contents. We believe how to incorporate the safety constraints into the generative models is still a valuable open question.

\end{document}